\documentclass[final,onefignum,onetabnum]{siamart220329}

\usepackage{amsfonts}

\title{Dynamical systems' based neural networks}

\author{Elena Celledoni\thanks{Department of Mathematical Sciences, NTNU, N-7491 Trondheim, Norway (\email{elena.celledoni@ntnu.no}, \email{davide.murari@ntnu.no}, \email{brynjulf.owren@ntnu.no})}
\and Davide Murari \footnotemark[1]
\and Brynjulf Owren\footnotemark[1]
\and Carola-Bibiane Schönlieb\thanks{Department of Applied Mathematics and Theoretical Physics, University of Cambridge, Wilberforce Road, Cambridge CB3 0WA, UK. (\email{cbs31@cam.ac.uk}, \email{fs436@cam.ac.uk})}
\and Ferdia Sherry\footnotemark[2]}

\usepackage{amsopn}

\newcommand{\dint}{\, \mathrm d}
\newcommand{\bl}[1]{\textcolor{black}{#1}}

\newcommand{\red}[1]{{#1}}

\newcommand{\ddiff}{\mathrm d}
\DeclareMathOperator{\Lip}{Lip}
\DeclareMathOperator{\diag}{diag}

\makeatletter
\newcommand*{\addFileDependency}[1]{
  \typeout{(#1)}
  \@addtofilelist{#1}
  \IfFileExists{#1}{}{\typeout{No file #1.}}
}
\makeatother

\usepackage[utf8]{inputenc}
\usepackage{subcaption}
\usepackage{comment}

\date{}

\begin{document}
\maketitle
\begin{abstract}
Neural networks have gained much interest because of their effectiveness in many applications. However,
their mathematical properties are generally not well understood. If there is some underlying geometric structure inherent to the data or to the function to approximate, it is often desirable to take this into account in the design of the neural network. In this work, we start with a non-autonomous ODE and build neural networks using a suitable, structure-preserving, numerical time-discretisation. The structure of the neural network is then inferred from the properties of the ODE vector field. Besides injecting more structure into the network architectures, this modelling procedure allows a better theoretical understanding of their behaviour. We present two universal approximation results and demonstrate how to impose some particular properties on the neural networks. A particular focus is on 1-Lipschitz architectures including layers that are not 1-Lipschitz. These networks are expressive and robust against adversarial attacks, as shown for the CIFAR-10 and CIFAR-100 datasets.
\end{abstract}



\section{Introduction}
Neural networks have been employed to accurately solve many different tasks (see, e.g.,~\cite{arridge2019solving,bronstein2021geometric,raissi2019physics,jumper2021highly}). Indeed, because of their excellent approximation properties, ability to generalise to unseen data, and efficiency, neural networks are one of the preferred techniques for the approximation of functions in high-dimensional spaces.

In spite of this popularity, a substantial number of results and success stories in deep learning still rely on empirical evidence and more theoretical insight is needed. Recently, a number of scientific papers on the mathematical foundations of neural networks have appeared in the literature,~\cite{berner2021modern,weinan2017proposal,Saxe2014ExactST,ShwartzZiv2017OpeningTB,thorpe2018deep,huang2020deep}. In a similar spirit, many authors consider the design of deep learning architectures taking into account specific mathematical properties such as stability, symmetries, or constraints on the Lipschitz constant~\cite{lecun1989backpropagation,hertrich2021convolutional,galimberti2021hamiltonian,smets2022pde,chen2019symplectic,gomez2017reversible,trockman2021orthogonalizing,jin2020sympnets,wang2020orthogonal,zakwan2022robust}. Even so, the imposition of structure on neural networks is often done in an ad hoc manner, making 
the resulting input to output mapping $F:\mathcal{X}\rightarrow \mathcal{Y}$ 
hard to analyse. In this paper, we describe a general and systematic way to impose desired mathematical structure on neural networks leading to an easier approach to their analysis.\newline\newline 
There have been multiple attempts to formulate unifying principles for the design of neural networks. We hereby mention Geometric Deep Learning (see e.g.~\cite{bronstein2017geometric,bronstein2021geometric}), Neural ODEs (see e.g.~\cite{chen2018neural,massaroli2020dissecting,ruiz2021neural,xiao2018dynamical}), the continuous-in-time interpretation of Recurrent Neural Networks (see e.g.~\cite{rusch2021unicornn,chang2019antisymmetricrnn}) and of Residual Neural Networks (see e.g.~\cite{weinan2017proposal,lu2018beyond,celledoni2021structure,ruthotto2020deep,agrachev2021control}). In this work, we focus on Residual Neural Networks (ResNets) and build upon their continuous interpretation.\newline\newline 
Neural networks are compositions of parametric maps, i.e.\ we can characterise a neural network as a map $\mathcal{N} = f_{\theta_k}\circ \ldots \circ f_{\theta_1}:\mathbb{R}^n\rightarrow\mathbb{R}^m$, with $f_{\theta_i}:\mathbb{R}^{n_i}\rightarrow\mathbb{R}^{n_{i+1}}$ being the network layers. For ResNets the most important parametric maps are of the form 
\begin{equation}
x\mapsto f_{\theta_i}(x)= x + h\Lambda(\theta_i,x).
\label{eq:expEulerStep}
\end{equation}
The continuous-in-time interpretation of ResNets arises from the observation that if $n_i=n_{i+1}$, $f_{\theta_i}$ coincides with one $h-$step of the explicit Euler method applied to the non-autonomous ODE $\dot{x}(t) =\Lambda(\theta(t),x(t))$. 
In this work, we consider piecewise-autonomous systems, i.e.\ we focus on time-switching systems of the form
\begin{equation}\label{eq:switch}
\dot{x}(t) = f_{s(t)}(x(t)),\,\,s:[0,T]\rightarrow \{1,\ldots,N\},\,f_i\in\mathcal{F},
\end{equation}
with $s$ being piecewise constant (see e.g.~\cite{ruiz2021neural,liberzon2003switching}), and $\mathcal{F}$ a family of parametric vector functions. This simplification is not restrictive and can help analyse and design neural networks, as we will clarify throughout the paper.

This interpretation of ResNets gives the skeleton of our reasoning. 
Indeed, we replace the explicit Euler method in~\cref{eq:expEulerStep} with suitable (structure-preserving) numerical flows of appropriate vector fields. We call the groups of layers obtained with these numerical flows ``dynamical blocks". The choice of the vector field is closely related to the structure to impose. For example, to derive symplectic neural networks we would apply symplectic time integrators to Hamiltonian vector fields. 
This approach enables us to derive new structured networks systematically and collocate other existing architectures into a more general setting, making their analysis easier. For instance, \cref{se:approx} presents a strategy to study the approximation capabilities of some structured networks. Finally, we highlight the flexibility and the benefits of this framework in~\cref{se:lipschitz}, where we show that to obtain expressive and robust neural networks, one can also include layers that are not 1-Lipschitz. \newline\newline 
There are multiple situations where one could be interested in networks with some prescribed property. We report three of them here, where we refer to $F$ as the function to approximate:
\begin{enumerate}
    \item When $F$ has some known characterising property, e.g.~$F$ is known to be symplectic; see~\cref{se:structure}.
    \item When the data we process has a particular structure, e.g.~vectors whose entries sum to one, as we present in~\cref{se:structure}.
    \item When we can approximate $F$ to sufficiently high accuracy with functions in $\mathcal{G}$, a space that is well suited to model the layers of a network. An example is using the space $\mathcal{G}$ of 1-Lipschitz functions to define a classifier robust to adversarial attacks; see~\cref{se:lipschitz}. 
 \end{enumerate}
Thus, there are various applications where having neural networks structured in a particular way is desirable. We will delve deeper into some of them in the following sections. To be precise, we remark that all the properties we focus on are preserved under composition, such as being 1-Lipschitz or symplectic.  \newline\newline 
The paper is structured in five sections. First, in~\cref{se:approx} we investigate the universal approximation capabilities of some neural networks, thanks to vector field decompositions, splitting methods and an embedding of the dynamics into larger dimensional spaces. We then move, in~\cref{se:lipschitz}, to a neural network that has the property of being 1-Lipschitz. After the mathematical derivation of the architecture, we present some numerical experiments on adversarial robustness for the CIFAR-10 and CIFAR-100 image classification problems. \bl{We devote a significant part of the experimental side of this paper to examples in the well-established field of adversarial robustness, but we furthermore provide examples of other desirable structural properties that can be imposed on neural networks using connections to dynamical systems.} In~\cref{se:structure}, we introduce such neural networks with specific designs. This last section aims to present in a systematic way how one can impose certain properties on the architecture. We finally conclude the paper in~\cref{se:conclusion}, mentioning some promising directions for further work. \newline\newline 
Before moving on, we now report a numerical experiment that motivates our investigation of structured neural networks. The results highlight how imposing a structure does not have to degrade the approximation's quality considerably. Furthermore, this experiment suggests that not all the constraining strategies perform similarly\bl{, as we also highlight in \cref{se:lipschitz}}. Thus, a systematic process to impose structure is essential since it allows changing the architecture in a guided manner while preserving the property of interest.
\subsection{Classification of points in the plane}
We present a numerical experiment for the classification problem of the dataset in~\cref{fig:points}. We consider neural networks that are 1-Lipschitz, as in~\cref{se:lipschitz}. We define the network layers alternatingly as contractive flow maps, whose vector fields belong to $\mathcal{F}_c = \{ -A^T\Sigma(Ax+b):\,A^TA=I\}$, and as flows  of other Lipschitz continuous vector fields in $\mathcal{F}=\{\Sigma(Ax+b):\,A^TA=I\}$, with $\Sigma(z) = [\sigma(z_1),\ldots,\sigma(z_n)]$ and $\sigma(s)=\max\{s,s/2\}$\footnote{To impose the weight orthogonality, we set $A=\mathrm{expm}(W-W^T)$ with $\mathrm{expm}$ being the matrix exponential and $W$ a trainable matrix.}. \bl{In \cref{se:lipschitz} we expand on the choice of this activation function $\sigma$, which is called LeakyReLU and was introduced in \cite{maas2013rectifier}.} \bl{The time steps for each vector field are network parameters, together with the matrices $A$ and vectors $b$.} We constrain the time steps to get a 1-Lipschitz network, see~\cref{se:lipschitz}. We report the results in~\cref{fig:alternation} and~\cref{ta:averages}.

The average classification test accuracy and final integration time, in combination, get better by combining $\mathcal{F}_c$ with $\mathcal{F}$ instead of considering $\mathcal{F}_c$ alone. In particular, we see that the final integration time $T$ with $\mathcal{F}_c \cup \mathcal{F}$ is the smallest without significantly compromising the accuracy. \bl{The parameter $T$ quantifies how much the network layers transform the points. The larger the timestep, the further a layer is from the identity map; hence we can get a more natural and efficient solution by alternating the vector fields.} In~\cref{se:approx} we reinforce this empirical result, proving results about theoretical approximation guarantees. This renders the possibility of obtaining neural networks with prescribed properties without compromising their approximation capabilities.
\begin{table}[ht!]
\centering
\begin{tabular}{|c|c|c|}
\hline
Adopted family of vector fields      & Median accuracy & Median of $T$ \\ \hline
$\mathcal{F}\cup \mathcal{F}_c$ & 98.0\%       & 1.84        \\ \hline
$\mathcal{F}$   & 99.0\%       & 7.53        \\ \hline
$\mathcal{F}_c$ & 97.3\%       & 19.82       \\ \hline
\end{tabular}
\caption{We perform 100 experiments alternating vector fields in $\mathcal{F}_c$ with those in $\mathcal{F}$, 100 using just vector fields in $\mathcal{F}_c$, and 100 with only those in $\mathcal{F}$. We work with networks with ten residual layers throughout the experiments. In the table we report the median final time $T$ and test accuracy for the three set of experiments analysed}
\label{ta:averages}
\end{table}
\begin{figure}[ht!]
\centering
\begin{subfigure}{.5\textwidth}
    \centering
    \includegraphics[width=.9\textwidth]{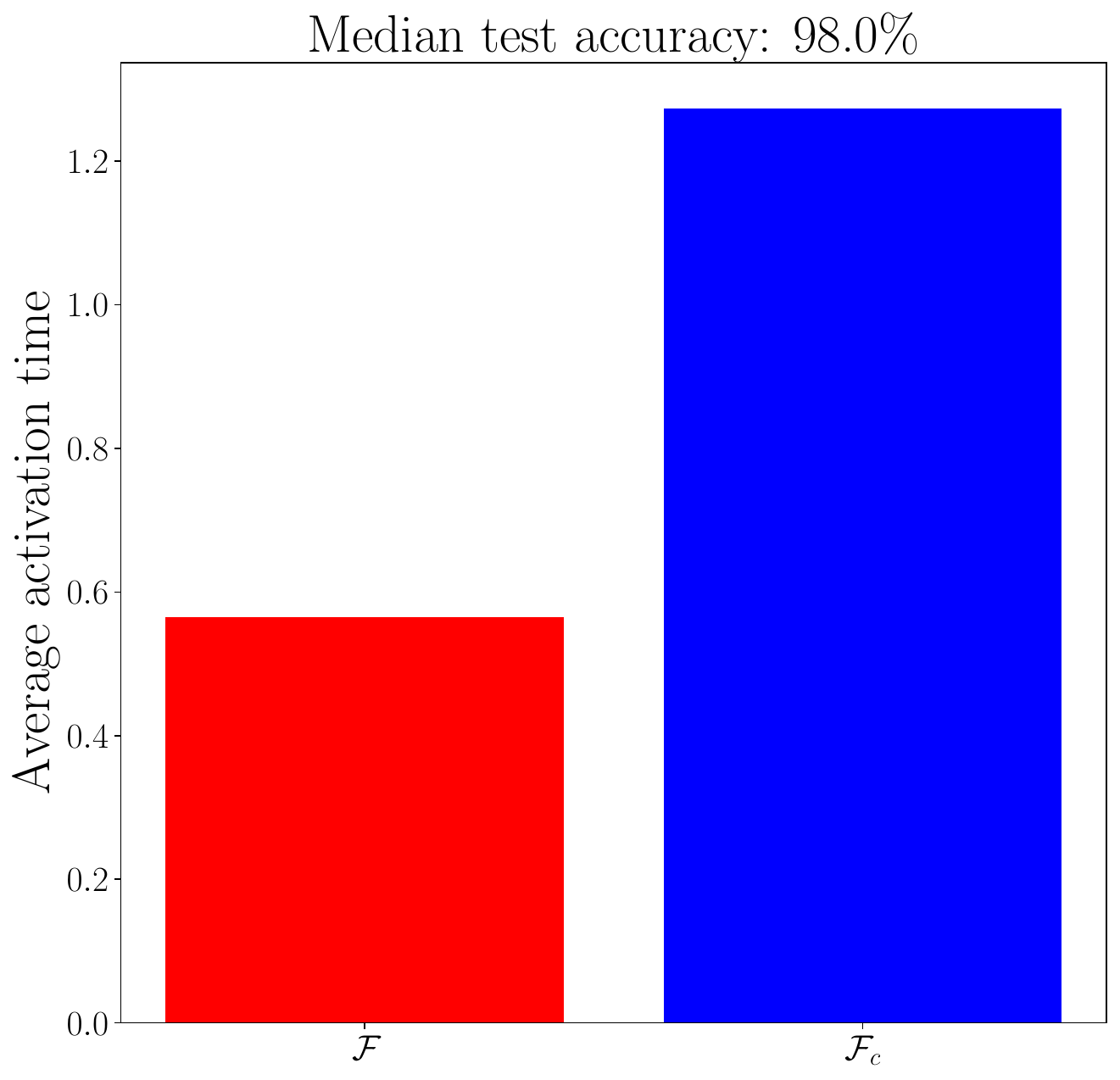}
    \caption{Mean length of the time intervals along which the two dynamical regimes are active.}
    \label{fig:alternation}
\end{subfigure}
\hfill
\begin{subfigure}{.45\textwidth}
    \centering 
    \includegraphics[width=.9\textwidth]{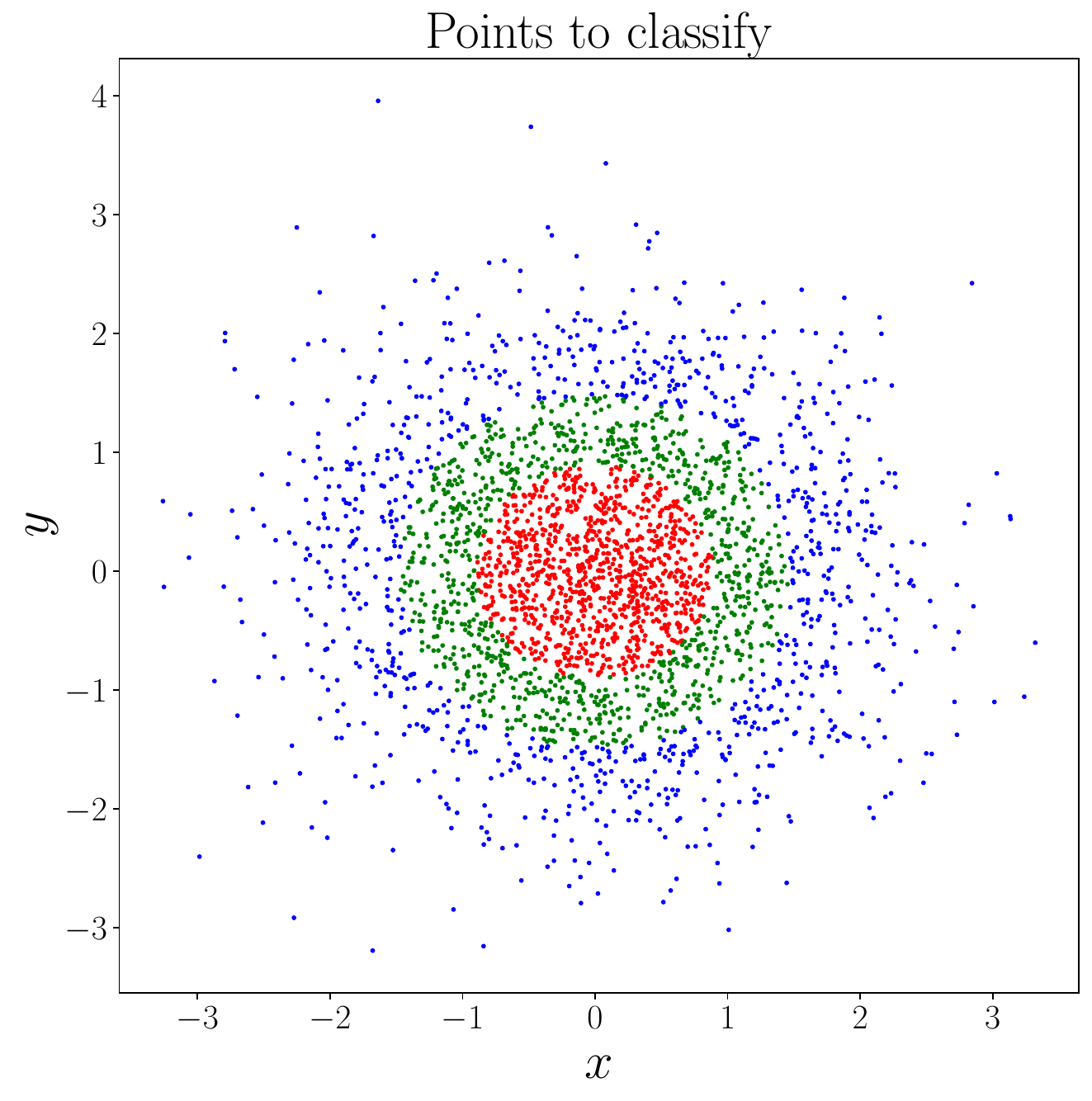}
    \caption{Dataset studied for the classification problem}
    \label{fig:points}
\end{subfigure}
\caption{Results from the experiments alternating the vector fields of $\mathcal{F}$ and those of $\mathcal{F}_c$, together with the dataset of interest}
\end{figure}
\section{Universal approximation properties}\label{se:approx}
As introduced before, neural network layers can be modelled by discretising ordinary differential equations. In particular, this ODE-based approach can also be beneficial for imposing some structure on neural networks and providing a better theoretical understanding of their properties. In this section, we follow this principle and prove the universal approximation capabilities of two neural network architectures. Starting with the continuous-in-time interpretation of neural networks, many approaches are possible to prove such approximation properties, often based on merging the broad range of results from dynamical systems theory and numerical analysis. One can proceed, for example, in a constructive way as done in~\cite{ruiz2021neural}, where the authors investigate the dynamics of some neural networks and explicitly construct solutions to the problem of approximating a target function. Another possibility is to study the approximation capabilities of compositions of flow maps, as done in~\cite{li2022deep}. In this section, we focus on two solutions that, to the best of our knowledge, are new. The first result is based on a vector field decomposition, while the second is based on embedding vector fields into larger dimensional spaces.\newline\newline
The approximation results that we cover rely on approximating vector fields arbitrarily well. Consequently, this allows us to approximate their flow maps accurately. This is based on the fact that for a sufficiently regular vector field $X\in\Lip(\mathbb{R}^n,\mathbb{R}^n)$, if $\tilde{X}:\mathbb{R}^n\rightarrow\mathbb{R}^n$ is such that for every $x\in\mathbb{R}^n$
\begin{equation}\label{eq:supapprox}
\|X(x)-\tilde{X}(x)\|<\varepsilon,
\end{equation}
then also their flow maps are close to one another for finite time intervals. We formalise this reasoning in~\cref{pr:prop}. In the proposition and throughout the paper, we denote with $\Phi^t_X(z)$ the time-$t$ flow of the vector field $X$, applied to $z$.
\begin{proposition}\label{pr:prop}
Let $X\in\Lip(\mathbb{R}^n,\mathbb{R}^n)$ and $\tilde{X}\in\Lip(\mathbb{R}^n,\mathbb{R}^n)$ be as in~\cref{eq:supapprox}. Then $\|\Phi_X^t(x)-\Phi_{\tilde{X}}^t(x)\|\leq \varepsilon t \exp{\left(\Lip(X)t\right)},$
where $\Lip(X)$ is the Lipschitz constant of $X$.
\end{proposition}
\begin{proof}
We consider the integral equations associated to the ODEs $\dot{x}(t) = X(x(t))$ and $\dot{\tilde{x}}(t) = \tilde{X}(\tilde{x}(t))$ and study the difference of their solutions both with the same initial condition $x\in\mathbb{R}^n$

\begin{align*}
\|\Phi^t_X(x)-\Phi^t_{\tilde{X}}(x)\| &=\left\| x + \int_0^t X( \Phi_X^s(x))\dint s - x - \int_0^t \tilde{X}(\Phi_{\tilde{X}}^s(x))\dint s \right\|\\
&\leq \int_0^t \left\| X(\Phi_X^s(x))-\tilde{X}(\Phi_{\tilde{X}}^s(x)) \right\|\dint s\\
&= \int_0^t \left\| X(\Phi_X^s(x))- X(\Phi_{\tilde{X}}^s(x))+ X(\Phi_{\tilde{X}}^s(x)) - \tilde{X}(\Phi_{\tilde{X}}^s(x)) \right\|\dint s \\
&\leq\Lip(X)\int_0^t\|\Phi_X^s(x)-\Phi_{\tilde{X}}^s(x)\|\dint s + \varepsilon t.
\end{align*}
Then we  conclude that
\[
\|\Phi^t_X(x)-\Phi^t_{\tilde{X}}(x)\|\leq \varepsilon t \exp{\left(\Lip(X)t\right)}.
\]
applying Gronwall's inequality.
\end{proof}
A particular consequence of this proposition is that if for every $\varepsilon>0$ there is an $\tilde{X}\in\mathcal{F}$ making~\cref{eq:supapprox} true, then we can approximate the $T-$flow map of $X$ arbitrarily well using elements of $\mathcal{F}$:
\[
\|\Phi^T_X(x)-\Phi^T_{\tilde{X}}(x)\|\leq \varepsilon T \exp(\Lip(X)T) = c\varepsilon .
\]
Because of this result, we now derive two approximation results for neural networks working at the level of modelling vector fields.
\subsection{Approximation based on a vector field decomposition}
We now aim to show that, for a particularly designed neural network, we can approximate arbitrarily well any continuous function in the $L^p$ norm and any differentiable invertible function in the supremum norm on compact sets. We also mention how to extend this last result to generic continuous functions.

\begin{theorem}\label{thm:switch}
Let $F:\Omega\subset \mathbb{R}^n\rightarrow\mathbb{R}^n$ be a continuous function, with $\Omega\subset\mathbb{R}^n$ a compact set. Suppose that it can be approximated, with respect to some norm $\|\cdot\|$, by a composition of flow maps of $\mathcal{C}^1(\Omega,\mathbb{R}^n)$ vector fields, i.e.\ for any $\varepsilon>0$, $\exists f_1,\ldots,f_k\in \mathcal{C}^1(\Omega,\mathbb{R}^n)$, such that
\begin{equation}\label{eq:hyp}
\|F - \Phi_{f_k}^{h_k}\circ \ldots \circ \Phi_{f_1}^{h_1}\| < \varepsilon.
\end{equation}
Then, $F$ can be approximated arbitrarily well by composing flow maps of gradient and sphere-preserving vector fields, i.e. $\|F - \Phi_{\nabla U^k}^{h_k}\circ \Phi_{X_S^k}^{h_k} \circ \ldots \circ \Phi_{\nabla U^1}^{h_1}\circ \Phi_{X_S^1}^{h_1}\|<\varepsilon$.
\end{theorem}
By sphere-preserving vector field, we mean a vector field $X_S$ having $z^Tz$ as a first integral, i.e.\ such that $z^TX_S(z)=0$ for any $z\in\mathbb{R}^n$. 

The norm $\|\cdot\|$ in~\cref{eq:hyp} can be any norm that is well defined for functions in $\mathcal{C}^1(\Omega,\mathbb{R}^n)$. Two typical choices in the literature are $L^p$ norms and the supremum norm
\begin{equation}\label{eq:convOmega}
\|F - \Phi_{f_k}^{h_k}\circ \ldots \circ \Phi_{f_1}^{h_1}\|:=\sup_{x\in\Omega}\|F(x)-\Phi_{f_k}^{h_k}\circ \ldots \circ \Phi_{f_1}^{h_1}(x)\|.
\end{equation}
Various works, like~\cite{brenier2003p} and~\cite{li2022deep}, have already proven the existence of vector fields $f_1,\ldots,f_k$ making~\cref{eq:hyp} true when $\|\cdot\|$ is the $L^p$ norm and $F$ is a continuous function. Regarding the validity of hypothesis~\cref{eq:hyp} with the norm defined in~\cref{eq:convOmega}, we mention~\cite{teshima2020universal} where the authors have proven that if $F$ is a smooth invertible map with smooth inverse, then the existence of $f_1,\ldots,f_k$ can be guaranteed.

\Cref{thm:switch} is a consequence of the Presnov decomposition of vector fields, introduced in~\cite{presnov2002non}, applied to the $k$ vector fields $f_1,\cdots,f_k\in \mathcal{C}^1(\Omega,\mathbb{R}^n)$ in~\cref{eq:hyp}. The Presnov decomposition is indeed a global decomposition of $\mathcal{C}^1(\mathbb{R}^n,\mathbb{R}^n)$ vector fields into the sum of a gradient and a sphere-preserving vector field. We now prove Theorem~\cref{thm:switch}, and we specialise it to the subfamilies of vector fields we implement to define neural networks.
\begin{proof}
The vector fields $f_1,\ldots,f_k$ are supposed to be continuously differentiable. Thus, they all admit a unique Presnov decomposition, i.e.\ they can be written as
\[
f_i(x) = \nabla U^i(x) + X^i_S(x),
\]
for a scalar function $U_i:\mathbb{R}^n\rightarrow\mathbb{R}$, with $U_i(0)=0$, and a sphere-preserving vector field $X^i_S$. In general the two vector fields $\nabla U^i(x)$ and $X_S^i(x)$ do not commute, i.e.\ the Jacobi-Lie bracket $[\nabla U^i,X_S^i]$ is not identically zero. However, because of the Baker-Campbell-Hausdorff formula (see e.g.~\cite[Section III.4.2]{geomBook}), as in splitting methods (see e.g.~\cite{mclachlan2002splitting}) we can proceed with an approximation of the form
\[
\Phi^h_{f_i} = \Phi^{h}_{\nabla U^i} \circ \Phi^{h}_{X_S^i}+ \mathcal{O}(h^2).
\] 
\bl{This last equality is the local error of the Lie Trotter splitting: local order 2 and global order 1 under the hypothesis that guarantees convergence\footnote{\bl{We prove the convergence of the Lie-Trotter splitting formula for Lipschitz regular vector fields in section \cref{se:convSplitting} of the supplementary material. Such proof extends similarly to other splitting strategies.}}. We recall that $\Phi_{f_i}^{h} = \Phi_{f_i}^{h/n}\circ \ldots \Phi_{f_i}^{h/n}$, where the flow maps are composed $n$ times}. Thus, up to choosing $n$ large enough, we can approximate as accurately as desired $\Phi_{f_i}^h$ with the composition of flow maps of sphere-preserving and gradient vector fields. This concludes the proof.
\end{proof}
Similar reasoning can be extended to other vector field decompositions, e.g. the Helmholtz decomposition, as long as $f_1,\ldots,f_k$ admit such a decomposition. In~\cref{se:lipschitz}, we adopt gradient vector fields whose flow maps expand and contract distances to obtain 1-Lipschitz neural networks. We now specialise~\cref{thm:switch} to the vector fields we use to model such neural networks. 
\begin{corollary}\label{co:coro}
Consider the same assumptions of~\cref{thm:switch}, and in particular the inequality~\cref{eq:hyp}. Then, we can approximate $F$ arbitrarily well by composing flow maps of expansive, contractive and sphere-preserving vector fields.
\end{corollary}
We first remark that with an expansive vector field we mean a vector field $X$ such that $\|\Phi^t_X(x)-\Phi^t_X(y)\|>\|x-y\|$ for any $t>0$, while by contractive we mean that $\|\Phi^t_X(x)-\Phi^t_X(y)\|<\|x-y\|$. To prove the corollary, we rely on a classical universal approximation theorem with non-polynomial activation functions (see e.g.~\cite{pinkus1999approximation}). For clarity, we report it here.
\begin{theorem}[Universal approximation, \cite{pinkus1999approximation}]\label{thm:uni}
Let $\Omega\subset\mathbb{R}^n$ be a compact set and $U\in\mathcal{C}^1(\mathbb{R}^n)$. Assume $\gamma\in\mathcal{C}^1(\mathbb{R})$ and $\gamma$ is not a polynomial. Then for every $\varepsilon>0$ there is 
\[
\tilde{U}(x) = \boldsymbol{\alpha}^T\Gamma(Ax+b),\quad \Gamma(z) = [\gamma(z_1),\ldots,\gamma(z_n)],
\] 
such that $\sup_{x\in\Omega}\|\tilde{U}(x)-U(x)\|<\varepsilon$,
and 
$
\sup_{x\in\Omega} \|\nabla \tilde{U}(x)-\nabla U(x)\|<\varepsilon.$
\end{theorem}
We now prove~\cref{co:coro}.
\begin{proof}
The proof follows the same reasoning of the one of~\cref{thm:switch}. Indeed, we first decompose each of the $f_1,\ldots,f_k$ of equation~\cref{eq:hyp} via the Presnov decomposition as $f_i(x) = \nabla U^i(x)+X_S^i(x)$. Then, we approximate each of the $U^i$ functions thanks to~\cref{thm:uni}. To ease the notation, we focus on one of the $f_i$ and denote it with $f$ from now on in the proof.

Let $U:\mathbb{R}^n\rightarrow\mathbb{R}$ and $X_S$ be so that $f(x) = \nabla U(x) + X_S(x)$. Choose then $\sigma(x) = \max\{ax,x\}$, $a\in (0,1)$, and $\gamma(x) = \int_0^x \sigma(s)\dint s$. Since $\gamma$ is not a polynomial and it is continuously differentiable, \cref{thm:uni} for any $\varepsilon>0$ ensures the existence of a function 
\[
\tilde{U}(x) = \boldsymbol{\alpha}^T\Gamma(Ax+b),
\]
that satisfies $\sup_{x\in\Omega}\|U(x)-\tilde{U}(x)\|<\varepsilon$ and $\sup_{x\in\Omega}\|\nabla U(x)-\nabla \tilde{U}(x)\|<\varepsilon$. We now split $\nabla \tilde{U}(x) = A^T\diag(\boldsymbol{\alpha})\Sigma(Ax+b)$ into a contractive and an expansive part, exploiting the two following properties of $\sigma$ and $\gamma$:
\begin{enumerate}
\item $\sigma$ is positively homogeneous, i.e.\ $\sigma(\lambda s) = \lambda\sigma(s)$ for $\lambda,s\in\mathbb{R}$, $\lambda\geq 0$,
\item $\gamma$ is strongly convex.
\end{enumerate}
We decompose $\boldsymbol{\alpha}$  as $\boldsymbol{\alpha}^+-\boldsymbol{\alpha}^-$, where $(\boldsymbol{\alpha}^+)_{k} = \max\{0,\alpha_k\}$, $(\boldsymbol{\alpha}^-)_k = -\min\{0,\alpha_k\}$ with $k=1,\ldots,n$. Because of the positive homogeneity, $\nabla \tilde{U}(x)$ can be rewritten as
\[
\nabla \tilde{U}(x) =A_1^T\Sigma(A_1x+b_1) - A_2^T\Sigma(A_2x+b_2) = X_E(x) + X_C(x)
\]
where
\[
A_1 = \diag(\boldsymbol{\alpha}^+)^{\frac{1}{2}}A,\,A_2= \diag(\boldsymbol{\alpha}^-)^{\frac{1}{2}}A ,\,b_1 = \diag(\boldsymbol{\alpha}^+)^{\frac{1}{2}}b,\,b_2 = \diag(\boldsymbol{\alpha}^-)^{\frac{1}{2}} b.
\]
Because of the strong convexity of $\gamma$, we have
\[
\frac{1}{2}\frac{\ddiff}{\ddiff t}\left\|z(t)-y(t)\right\|^2 = \langle X_E(z(t))-X_E(y(t)),z(t)-y(t)\rangle > 0
\]
with $z(t) = \Phi_{X_E}^t(z_0)$ and $y(t) = \Phi_{X_E}^t(y_0)$. This means that the flow of $X_E$ is an expansive map. A similar reasoning shows that $X_C$ has a contractive flow map. We can now conclude as in~\cref{thm:switch} since we have shown that every $f_i$ in~\cref{eq:hyp} can be approximated arbitrarily well as $f_i(x) \approx X_E^i(x) + X_C^i(x) + X_S^i(x)$.
\end{proof}
As for the expansive and the contractive vector fields, to define neural networks based on~\cref{co:coro} one needs to parameterise the vector field $X_S^i(z)$ that preserves spheres. Many possibilities are available, and we report a couple of them. The first is
\[
\tilde{X}_S(z) = P(z)B^T\Sigma(Cz+d),\,\,B,C\in\mathbb{R}^{m\times n},\,d\in\mathbb{R}^m,\, P(z) = I_n - \frac{zz^T}{\|z\|^2},
\]
where $P(z):T_z\mathbb{R}^n\rightarrow T_zS^2_{\|z\|}$ is the orthogonal projection on the space $\langle z\rangle^{\perp}$ and $I_n\in\mathbb{R}^{n\times n}$ is the identity matrix. Another option is
\[
\tilde{X}_S(z) = \Lambda(z,\theta)z
\] 
where $\Lambda(z,\theta) = A(z,\theta) - A(z,\theta)^T\in\mathbb{R}^{n\times n}$ with $A$ being a strictly upper triangular matrix whose entries are modelled by $B\Sigma(Cx+b)\in\mathbb{R}^N$, $B\in\mathbb{R}^{N\times m}$, $C\in\mathbb{R}^{m\times n}$, $b\in\mathbb{R}^m$, $N= \frac{n(n-1)}{2}$. These two possibilities allow us to approximate any sphere-preserving vector field arbitrarily because of classical universal approximation results, like the one mentioned in~\cref{thm:uni}. We prefer, for practical reasons, the second one in the experiments reported in~\cref{se:expReg} of the supplementary material. 

We now summarise the results presented in the context of neural networks. Suppose that $\|F - \Phi_{f_k}^{h_k} \circ \ldots \circ \Phi_{f_1}^{h_1}\|<\varepsilon $
and that $f_i \approx \tilde{f}_i = X^i_C + X^i_E + \tilde{X}_S^i$ for $i=1,\ldots,k$. In~\cref{thm:switch} we have worked with the exact flows of the vector fields. However, most of the times these are not available and hence a numerical approximation is needed. This is exactly equivalent to applying a splitting numerical integrator (see e.g.~\cite[Chapter 2]{geomBook} or~\cite{mclachlan2002splitting}) to approximate the $h_i$-flow map of $\tilde{f}_i$ (and hence also of $f_i$) and get 
\begin{equation}\label{eq:nn}
F(x) \approx \mathcal{N}(x) = \Psi^{h_k}_{X^k_C} \circ \Psi^{h_k}_{X^k_E} \circ \Psi^{h_k}_{X_S^k} \circ \ldots \circ \Psi^{h_1}_{X^1_C} \circ \Psi^{h_1}_{X^1_E} \circ \Psi^{h_1}_{X_S^1}(x).
\end{equation}
Here we denote with $\Psi^h_f$ a discrete approximation of the exact flow $\Phi^h_f$ and $\mathcal{N}$ is the neural network that approximates the target function $F$. Because of~\cref{co:coro} and basic theory of numerical methods for ODEs, $\mathcal{N}$ hence can approximate arbitrarily well $F$ in the norm $\|\cdot\|$. \newline\newline 
We remark that the neural network $\mathcal{N}$ defined in~\cref{eq:nn} does not change the dimensionality of the input point, i.e.\ it is a map from $\mathbb{R}^n$ to itself. However, usually ResNets allow for dimensionality changes thanks to linear lifting and projection layers. One can extend all the results presented in this section to the dimensionality changes, where instead of defining the whole network as the composition of discrete flow maps, just the ``dynamical blocks" are characterised in that way, as represented in~\cref{fig:changeDim}.
\begin{figure}[ht!]
    \centering
    \includegraphics[trim={0 0 5cm 0},clip,width=.9\textwidth]{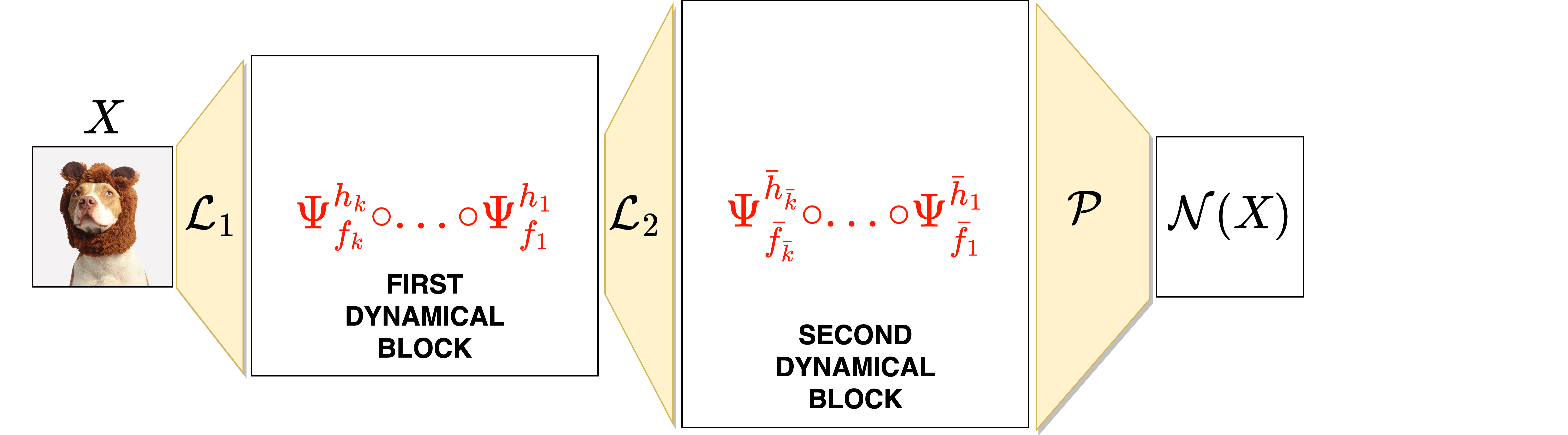}
    \caption{Representation of a ResNet made of two dynamical blocks, two lifting layers $\mathcal{L}_1$, $\mathcal{L}_2$ and a final projection layer $\mathcal{P}$.}
    \label{fig:changeDim}
\end{figure}
Consequently, one can extend the results presented in~\cite{zhang2020approximation}. In particular, one can show that by composing flow maps of sphere-preserving and gradient vector fields, generic continuous functions can also be approximated in the sense of~\cref{eq:convOmega}, \bl{as long as linear lifting and projection layers are allowed in the network}. \newline\newline 
In~\cref{se:expReg} of the supplementary material we show some numerical experiments where some unknown dynamical systems and some target functions are approximated starting from the above results. We now introduce another way to get expressivity results starting from the continuous-in-time interpretation of neural networks.
\subsection{Approximation based on Hamiltonian vector fields}
Augmenting the dimensionality of the space where the dynamics is defined is a typical technique for designing deep neural networks, see~\cref{fig:changeDim}. Based on this idea, we now study the approximation properties of networks obtained by composing flow maps of Hamiltonian systems. For an introductory presentation of Hamiltonian systems, see~\cite{leimkuhler2004simulating}. 

We now show that for any function $F$ for which hypothesis~\cref{eq:hyp} holds, one can approximate $F$ arbitrarily well, in the same function norm, by composing flow maps of Hamiltonian systems and linear maps. Consequently, symplectomorphisms like those defined by SympNets (\cite{jin2020sympnets}) can also be used approximate $F$ arbitrarily well.

This result relies on the embedding of a vector field $f\in\mathcal{C}^1(\mathbb{R}^n,\mathbb{R}^n)$ into a Hamiltonian vector field on $\mathbb{R}^{2n}$. To do so, we first define the linear map $
L:\mathbb{R}^n\rightarrow \mathbb{R}^{2n}$, as $z\mapsto L(z)=(z,0).$
We then introduce the function $H_f(z,p) = p^Tf(z)$, where $p$ is the conjugate momentum of $z$. The gradient of such a function is
\[
\nabla H_f(z,p) = \begin{bmatrix} \frac{\partial \left[ p^Tf(z)\right]}{\partial z} \\ f(z) \end{bmatrix}.
\]
This implies that the Hamiltonian ODEs associated to $H_f$ are 
\[
\begin{bmatrix}
\dot{z}\\
\dot{p}
\end{bmatrix} = X_{H_f}(z,p) = \mathbb{J}\nabla H_f(z,p) = \begin{bmatrix}
f(z) \\ -\frac{\partial  \left[ p^Tf(z)\right]}{\partial z}
\end{bmatrix}.
\]
Hence, we have $\Phi_{f}^h = P \circ \Phi_{X_{H_{f}}}^h \circ L $
where $P$ is the projection on the first component $\mathbb{R}^{2n}\ni (z,p)\mapsto z\in\mathbb{R}^n$. 
\bl{This construction, together with hypothesis~\cref{eq:hyp}, implies that 
\[
\|F - P\,\circ\, \Phi_{X_{H_{f_k}}}^{h_k}\,\circ\,L\,\circ\,P\,\circ\,\Phi_{X_{H_{f_{k-1}}}}^{h_{k-1}}\,\circ\,L \circ \ldots L\,\circ \,P \,\circ\,\Phi_{X_{H_{f_1}}}^{h_1}\,\circ\,L\|<\varepsilon.
\]
}
One could be interested in such a lifting procedure and hence work with Hamiltonian systems because discretising their flow maps with symplectic methods might generate more stable networks or, as highlighted in~\cite{galimberti2021hamiltonian}, could prevent vanishing and exploding gradient problems.

\section{Adversarial robustness and Lipschitz neural networks}\label{se:lipschitz}
In this section, we consider the problem of classifying points of a set $\mathcal{X}\subset\mathbb{R}^n$. More precisely, given a set $\mathcal{X}=\cup_{i=1}^C \mathcal{X}_i$ defined by the disjoint union of $C$ subsets $\mathcal{X}_1,\ldots,\mathcal{X}_C$, we aim to approximate the function $\ell:\mathcal{X}\rightarrow \{1,\ldots,C\}$ that assigns all the points of $\mathcal{X}$ to their correct class, i.e.\ $\ell(x) = i$ for all $x\in\mathcal{X}_i$ and all $i=1,\ldots,C$. Because of their approximation properties, one can often choose neural networks to solve classification problems, i.e.\ as models that approximate the labelling function $\ell$. On the other hand, there is increasing evidence that trained neural networks are sensitive to well-chosen input perturbations called adversarial attacks. The first work that points this out is~\cite{Szegedy2014IntriguingPO} and, since then, numerous others (see e.g.~\cite{madry2018towards,carlini2017towards,Goodfellow2015ExplainingAH}) have introduced both new ways to perturb the inputs (attacks) and to reduce the sensitivity of the networks (defences). We first formalise the problem of adversarial robustness from the mathematical point of view and then derive a network architecture with inherent stability properties.

Let $\mathcal{N}:\mathbb{R}^n\rightarrow\mathbb{R}^C$ be a neural network trained so that the true labelling map $\ell$ is well approximated by $\hat{\ell}(x) = \arg\max_{i=1,\ldots,C}\mathcal{N}(x)_i$ for points $x\in\mathcal{X}$. Furthermore, let us assume 
\[
\|x-y\|\geq 2\varepsilon\quad \forall x,y\in\mathcal{X},\,\,\ell(x)\neq \ell(y)
\]
for some norm $\|\cdot\|:\mathbb{R}^n\rightarrow\mathbb{R}^+$ defined on the ambient space. With this setup, we say the network $\mathcal{N}$ is $\varepsilon-$robust if
\[
\ell(x) = \hat{\ell}(x) = \hat{\ell}(x+\delta),\quad \forall \delta\in\mathbb{R}^n,\,\, \|\delta\|< \varepsilon.
\]
In order to quantify the robustness of $\mathcal{N}$ we, first of all, consider its Lipschitz constant $\Lip(\mathcal{N})$, i.e.\ the smallest scalar value such that
\[
\|\mathcal{N}(x)- \mathcal{N}(y)\|_2\leq \Lip(\mathcal{N})\|x-y\|,\quad \forall x,y\in\mathbb{R}^n,
\]
where $\|\cdot\|_2:\mathbb{R}^C\rightarrow\mathbb{R}^+$ is the $\ell^2$ norm. We also need a way to quantify how certain the network predictions are. A measure of this certainty level is called margin in the literature (see e.g.~\cite{anil2019sorting,bartlett2017spectrally,tsuzuku2018lipschitz}) and it is defined as
\[
\mathcal{M}_{\mathcal{N}}(x) = \mathcal{N}(x)^Te_{\ell(x)} - \max_{j\neq \ell(x)}\mathcal{N}(x)^Te_j,
\]
where $e_{i}$ is the $i-$th vector of the canonical basis of $\mathbb{R}^C$. Combining these two quantities, in~\cite{tsuzuku2018lipschitz} the authors show that if the norm $\|\cdot\|$ considered for $\mathcal{X}$ is the $\ell^2$ norm of the ambient space $\mathbb{R}^n$, then
\begin{equation}
\mathcal{M}_{\mathcal{N}}(x)\geq \sqrt{2}\varepsilon \Lip(\mathcal{N})\implies \mathcal{M}_{\mathcal{N}}(x+\delta)\geq 0\quad\forall \delta\in\mathbb{R}^n,\,\|\delta \|\leq \varepsilon.
\label{eq:marginIneq}
\end{equation}
Hence, for the points in $\mathcal{X}$ where~\cref{eq:marginIneq} holds, the network is robust to perturbations with a magnitude not greater than $\varepsilon$. This result can be extended to generic $\ell^p$ metrics, but, in this section, we focus on the case where $\|\cdot\|$ is the $\ell^2$ metric of $\mathbb{R}^n$ and, from now on, we keep denoting it as $\|\cdot\|$. 

Motivated by inequality~\cref{eq:marginIneq}, we present a strategy to constrain the Lipschitz constant of ResNets to the value of $1$. Differently from~\cite{celledoni2021structure,zakwan2022robust,meunier2021scalable}, we impose such a property on the network without relying only on layers that are 1-Lipschitz. This strategy relies on the ODE-based approach that we are presenting and is motivated by the interest of getting networks that also have good expressivity capabilities. Indeed, we remark that in~\cref{se:approx} we studied the approximation properties of networks similar to those we consider in this section. We conclude the section with extensive numerical experiments for the adversarial robustness with the CIFAR-10 and CIFAR-100 datasets to test the proposed network architectures.
\subsection{Non-expansive dynamical blocks}
Consider a scalar differentiable function $V:\mathbb{R}^n\rightarrow \mathbb{R}$ \bl{that is also strongly convex, i.e. it admits a $\mu>0$ such that
\begin{equation}
\langle \nabla V(x)-\nabla V(y),x-y\rangle \geq \mu \|x-y\|^2,
\label{eq:gradStrongConvex}
\end{equation}
see e.g. \cite[Chapter 6]{hiriart2013convex}. We refer to a function $V$ that is strongly-convex with strong convexity constant $\mu$ as $\mu$-strongly convex.} This said, it follows that the dynamics defined by the ODE 
\begin{equation}\label{eq:contractiveGradFlow}
\dot{x}(t) = -\nabla V(x(t))=X(x(t))
\end{equation}
is contractive, since
\begin{equation}\label{eq:ineq}
\begin{split}
\frac{1}{2}\frac{\ddiff}{\ddiff t}\|x(t)-y(t)\|^2 &= -\langle x(t)-y(t),\nabla V(x(t)) - \nabla V(y(t))\rangle \leq -\mu\|x(t)-y(t)\|^2 \\
\implies \|x(t)-y(t)\|&\leq e^{-\mu t}\|x_0-y_0\|<\|x_0-y_0\|\,\,\forall\,t\geq 0,
\end{split}
\end{equation}
where $x(t)=\Phi^t_X(x_0)$ and $y(t)=\Phi^t_X(y_0)$. A choice for $V$ is $V(x) = \boldsymbol{1}^T\Gamma(Ax+b)$, where $\Gamma(x)=[\gamma(x_1),\ldots,\gamma(x_n)]$,  $\gamma:\mathbb{R}\rightarrow\mathbb{R}$ is a strongly convex differentiable function, and $\boldsymbol{1} = [1,\ldots,1]\in\mathbb{R}^n$. In this way, the network we generate by concatenating explicit Euler steps applied to such vector fields has layers of the type
\[
x\mapsto \Psi_X^h(x)=x - h A^T\Sigma(Ax+b)
\]
where $\Sigma(x) = [\sigma(x_1),\ldots,\sigma(x_n)]$ and $\sigma(s)=\gamma'(s)$.

If we discretise the ODE introduced above reproducing the non-expansive behaviour at a discrete level, as presented for example in~\cite{dahlquist1979generalized,celledoni2021structure,meunier2021scalable}, we get that the numerical flow $\Psi^h_X$ is non-expansive too. Consequently, we can obtain 1-Lipschitz neural networks composing these non-expansive discrete flow maps. A simple way to discretise~\cref{eq:contractiveGradFlow} while preserving non-expansiveness is to use explicit-Euler steps with a small enough step size. Indeed, assuming $\Lip(\sigma)\leq 1$, a layer of the form
\begin{equation}\label{eq:contractiveDyn}
x\mapsto x - hA^T\Sigma(Ax+b),\quad h\leq \frac{2}{\|A\|^2 },\,\,\|A\|=\sup_{\substack{x\in\mathbb{R}^n\\\|x\|=1}}\|Ax\|,
\end{equation}
is guaranteed to be 1-Lipschitz.
We remark that, as highlighted in~\cite{celledoni2021structure,meunier2021scalable}, it is not necessary to require strong convexity for $\gamma$ in order to make $\Phi^t_X$ 1-Lipschitz. Indeed, it is enough to take $\gamma$ convex. However, the strong convexity assumption allows us to include other layers that are not 1-Lipschitz thanks to inequality~\cref{eq:ineq}.

We now shortly present the reasoning behind this statement. Consider another ODE $\dot{x}(t) = Y(x(t))$ where $Y$ is again a vector field on $\mathbb{R}^n$ and suppose that $Y$ is $L$-Lipschitz. Then, we have that
\[
\|\Phi^{\bar{t}}_Y(x_0)-\Phi^{\bar{t}}_Y(y_0)\|\leq \exp{(L{\bar{t}})}\|x_0-y_0\|.
\]
This implies that, given $X$ as in~\cref{eq:contractiveGradFlow}, the map $\Phi_X^{t}\circ \Phi_Y^{\bar{t}} =: C_{{\bar{t}},t}$ satisfies
\begin{equation}\label{eq:nonexpansive}
\|\Phi^t_X(\Phi^{\bar{t}}_Y(x_0))-\Phi^t_X(\Phi^{\bar{t}}_Y(y_0))\|\leq \exp\left(-\mu t + L{\bar{t}}\right)\|x_0-y_0\|,
\end{equation}
so $C_{{\bar{t}},t}$ is Lipschitz continuous and will be 1-Lipschitz if $\exp{(-\mu t + L{\bar{t}})}\leq 1$. This amounts to imposing $L{\bar{t}}\leq \mu t$ on the considered vector fields and time intervals on which corresponding flow maps are active. The map $C_{{\bar{t}},t}$ can be seen as the exact $(t+{\bar{t}})-$flow map of the switching system having a piecewise constant (in time) autonomous dynamics. In particular, such a system coincides with $Y$ for the first time interval $[0,{\bar{t}})$ and with $X$ for the time interval $[{\bar{t}},{\bar{t}}+t)$. \newline\newline 
We could choose $Y$ to be the gradient vector field of an $L$-smooth scalar potential. In other words, we ask for its gradient to be $L$-Lipschitz. \bl{An option is hence
\[
Y(x) = A^T\Sigma(Ax+b),\,\,\|A\|\leq 1,
\]
with $\sigma$ that is $L-$Lipschitz.} Thus, one possible way of building a dynamical block of layers that is 1-Lipschitz is through a consistent discretisation of the switching system
\begin{equation}\label{eq:alternation}
\dot{x}(t) = (-1)^{s(t)}A_{s(t)}^T\Sigma(A_{s(t)}x(t)+b_{s(t)}),
\end{equation}
where $t\mapsto s(t)$ is a piecewise constant time-switching signal that, following~\cref{eq:nonexpansive}, balances the expansive and contractive regimes. In~\cref{eq:alternation}, we are assuming that $\Sigma(z) = [\sigma(z_1),\ldots,\sigma(z_n)]$ with $\sigma$ $1-$Lipschitz and $\gamma(s) = \int_0^s\sigma(t)dt$ strongly convex. In the numerical experiment we report at the end of the section we design $s(t)$ giving an alternation between contractive and possibly non-contractive behaviours. In the following subsection, we present two possible approaches to discretise numerically the system in~\cref{eq:alternation}, mentioning how this extends to more general families of vector fields.

In this subsection, we have worked to obtain dynamical blocks that are non-expansive for the Euclidean metric. In~\cref{se:generalmetric} of the supplementary material we show a way to extend this reasoning to more general metrics defined in the input space.
\subsection{Non-expansive numerical discretisation}\label{se:numDisc}
As presented in the introductory section, it is not enough to have a continuous model that satisfies some property of interest to get it at the network level. Indeed, discretising the solutions to such an ODE must also be done while preserving such a property. One approach that always works is to restrict the step sizes to be sufficiently small so that the behaviour of the discrete solution resembles the one of the exact solution. This strategy can lead to expensive network training because of the high number of time steps. On the other hand, this strategy allows weaker weight restrictions and better performances. We remark how this translates for the dynamical system introduced in~\cref{eq:alternation}, with $\sigma(x) = \max\{x,ax\}$. For that ODE, the one-sided Lipschitz constant of contractive layers is $\mu = a\lambda_{\min}(A^TA)$, $\lambda_{\min}$ being the smallest eigenvalue. Thus, if $A$ is orthogonal, we get $\mu = a$. Under the same orthogonality assumption, the expansive layers in~\cref{eq:alternation} have Lipschitz constant $L=1$ and this allows to specialise the non-expansiveness condition~\cref{eq:nonexpansive} to ${\bar{t}}\leq at$. Thus, if we impose such a relationship and perform sufficiently small time steps, also the numerical solutions will be non-expansive.\newline\newline 
However, frequently smarter choices of discrete dynamical systems can lead to leaner architectures and faster training procedures. We focus on the explicit Euler method for this construction, although one can work with other numerical methods, like generic Runge-Kutta methods, as long as the conditions we derive are adjusted accordingly. We concentrate on two time steps applied to equation~\cref{eq:alternation}, but then the reasoning extends to every pair of composed discrete flows and to other families of vector fields. Let
\begin{equation}
\begin{split}
    \tilde{\Psi}^{h_1}(x) &= x - h_1 A_c^T\Sigma(A_cx+b_c) =: x-h_1 X(A_c,b_c,x)\\
    \Psi^{h_2}(x) &= x + h_2 A_e^T\Sigma(A_ex+b_e) =: x+ h_2 X(A_e,b_e,x).
\end{split}
\label{eq:gradientAlternationRegime}
\end{equation}
\bl{We remark that here the subscripts $c$ and $e$ stand for contractive and expansive respectively.} The condition we want to have is that the map $F_{h}=\tilde{\Psi}^{h_1}\circ \Psi^{h_2}$ is 1-Lipschitz, or at least that this is true when $A_c$, $A_e$ and $\Sigma$ satisfy some well-specified properties. We first study the Lipschitz constant of both the discrete flow maps and then upper bound the one of $F_h$ with their product. We take two points $x,y\in\mathbb{R}^n$, define $\delta X(A_c,b_c,x,y) = X(A_c,b_c,y)-X(A_c,b_c,x)$, and proceed as follows

\begin{align*}
\|\tilde{\Psi}^{h_1}(y)-\tilde{\Psi}^{h_1}(x)\|^2 
&= \|y-x\|^2 - 2h_1\langle y-x, \delta X(A_c,b_c,x,y)\rangle +h_1^2\|\delta X(A_c,b_c,x,y)\|^2 \\
&\leq \|y-x\|^2 - 2h_1\lambda_{\min}(A_c^TA_c)a\|y-x\|^2 + h_1^2\|A_c\|^4\|y-x\|^2,
\end{align*}
\bl{where the last inequality is because we consider} $\sigma = \max\{ax,x\}$. In the experiments we present at the end of the section, we assume all the weight matrices to be orthogonal, hence $\lambda_{\min}(A_c^TA_c)=1$. Multiple works support this weight constraint as a way to improve the generalisation capabilities, the robustness to adversarial attacks and the weight efficiency (see e.g.~\cite{wang2020orthogonal,trockman2021orthogonalizing}). We will detail how we get such a constraint on convolutional layers in the numerical experiments.
\begin{figure}[ht!]
    \centering
\begin{subfigure}{.49\textwidth}
    \centering
    \includegraphics[width=.7\textwidth]{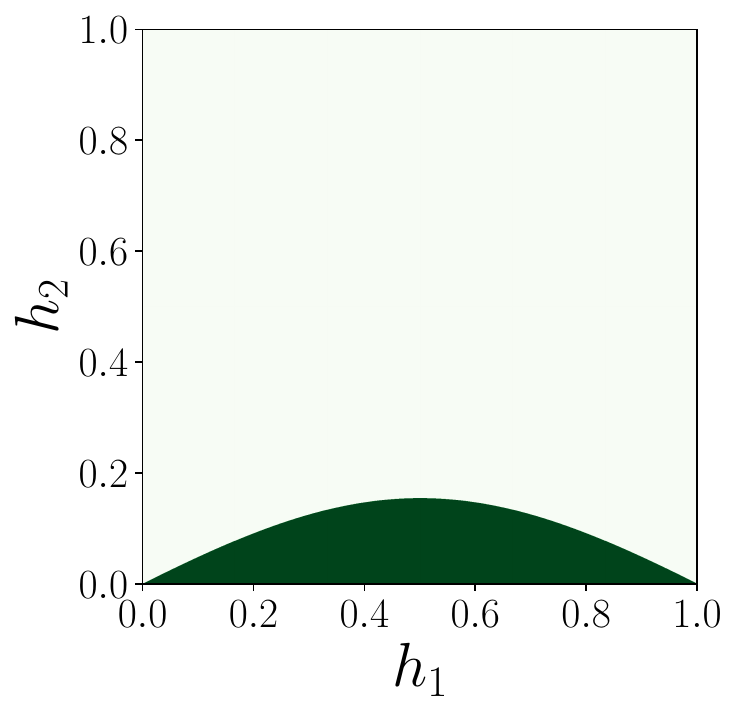}
    \caption{Case $S=1$.}
    \label{fig:regionS1}
\end{subfigure}
\hfill
\begin{subfigure}{.49\textwidth}
    \centering
    \includegraphics[width=.7\textwidth]{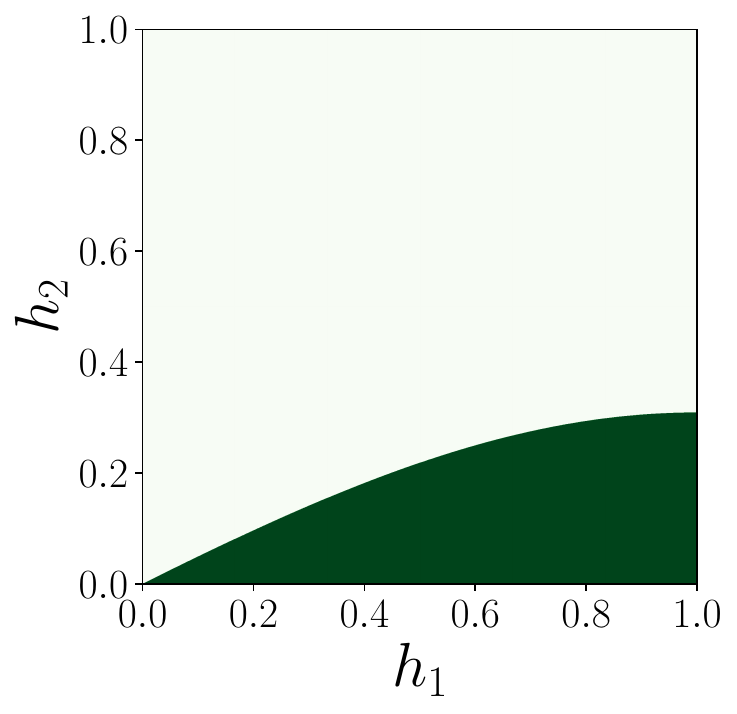}
    \caption{Case $S=2$.}
    \label{fig:regionS2}
\end{subfigure}
\caption{Representation of the non-expansiveness region~\cref{eq:contractiveRegion} for the choice $a=0.5$.}
\end{figure}

The orthogonality of $A_c$ implies $\|A_c\|=1$ and hence we get that a Lipschitz constant of $\tilde{\Psi}^{h_1}$ is $L_1 = \sqrt{1-2h_1a+h_1^2}$. For the expansive flow map $\Psi^{h_2}$, we have
\begin{equation}
\begin{split}
\|\Psi^{h_2}(y)-\Psi^{h_1}(x)\|&\leq \|x-y\|+ h_2\Lip\left(A_e^T\Sigma(A_ez+b_e)\right)\|x-y\|\\
&\leq (1+h_2)\|x-y\|
\end{split}
\label{eq:expLipschitz}
\end{equation}
under the orthogonality assumption for $A_e$. The same result holds also if we just have $\|A_e\|\leq 1$. This leads to a region in the $(h_1,h_2)-$plane where $L_1\cdot L_2\leq 1$ that can be characterised as follows
\begin{equation}\label{eq:contractiveRegion}
\mathcal{R} = \{(h_1,h_2)\in [0,1]^2:\;(1+h_2)\sqrt{1-2h_1a+h_1^2}\leq 1\}.
\end{equation}
This is represented in~\cref{fig:regionS1} for the case $a=0.5$, that is the one used also in the numerical experiment for adversarial robustness. Thus, we now have obtained a way to impose the 1-Lipschitz property on the network coming from the discretisation of the ODE~\cref{eq:alternation}. It is clear that the result presented here easily extends to different time-switching rules (i.e.\ a different choice of $s(t)$), as long as there is the possibility of balancing expansive vector fields with contractive ones. Furthermore, to enlarge the area in the $(h_1,h_2)-$plane where non-expansiveness can be obtained, one can decide to divide into sub-intervals the time intervals $[0,h_1]$ and $[0,h_2]$. Doing smaller steps, the allowed area increases. Indeed, instead of doing two single time steps of length $h_1$ and $h_2$, one can perform $S$ time-steps all of step-length $h_1/S$ or $h_2/S$. Thus, replacing $\bar{h}_1 = h_1/S$ and $\bar{h}_2=h_2/S$ into~\cref{eq:contractiveRegion} it is immediate to see that $h_1$ and $h_2$ are allowed to be larger than with the case $S=1$. For example, if we again fix $a=0.5$ and set $S=2$, we get the area represented in~\cref{fig:regionS2}.
The choice of $a=0.5$ and $S=2$ is the one we adopt in the experiments reported in this section. We now conclude the section showing how the derived architecture allows to improve the robustness against adversarial attacks for the problem of image classification.
\subsection{Numerical experiments with adversarial robustness}
We now apply the reasoning presented above to the problem of classifying images of the CIFAR-10 and CIFAR-100 datasets. The implementation is done with PyTorch and is available at the GitHub repository associated to the paper\footnote{\url{https://github.com/davidemurari/StructuredNeuralNetworks}}. We work with convolutional neural networks, \bl{and with the activation function $\sigma(x) = \max\left\{x,\frac{x}{2}\right\}$, if not otherwise specified. We test multiple architectures and start by introducing the one coming  directly from the derivation reported in the previous section. The residual layers of this network are dynamical blocks based on the discrete flow maps
\begin{align}\label{eq:resLayers}
    \tilde{\Psi}^{h_1}(x) &= x - h_1 A_c^T\Sigma(A_cx+b_c) =: x-h_1 X(A_c,b_c,x),\,\,A_c^TA_c=I \nonumber \\
    \Psi^{h_2}(x) &= x + h_2 A_e^T\Sigma(A_e x+b_e) =: x+ h_2 X(A_e,b_e,x),\,\,A_e^TA_e=I \nonumber \\
    x & \mapsto \Psi^{h_2/2} \circ \tilde{\Psi}^{h_1/2}\circ \Psi^{h_2/2} \circ \tilde{\Psi}^{h_1/2} (x).
\end{align}
}
\red{The orthogonality of the convolutional filters $A_c$ and $A_e$ is imposed through a regularisation strategy proposed in \cite{wang2020orthogonal}. We comment more on this and alternative strategies later on in the description of the experimental setup.}
The step restriction is imposed after every training iteration, projecting back the pairs $(h_1,h_2)$ in the region represented in~\cref{fig:regionS2} if needed.

\bl{The strategy in equation \eqref{eq:resLayers} is defined as a ``prescribed switching strategy" in the numerical experiments. It is applied both for the experiment on CIFAR-10 and CIFAR-100. To demonstrate the freedom one still has while using an alternation strategy to design the layers, we mention another switching strategy that we shall call ``flexible". In this case, we have the following alternation
\begin{align}\label{eq:flexible}
\tilde{\Psi}^{h_1}(x) &= x - h_1 A_c^T\Sigma(A_cx+b_c) =: x-h_1 X(A_c,b_c,x),\,\,A_c^TA_c=I \nonumber \\
\Psi^{h_2}(x) &= x + h_2A^T\mathrm{ReLU}(Ax+b) =: x+ h_2 X(A,b,x),\,\,A^TA=I \nonumber \\
x & \mapsto \tilde{\Psi}^{h_1/2}\circ \tilde{\Psi}^{h_1/2}\circ \Psi^{h_2}(x).
\end{align}
Here the weight $A$ is no longer with a subscript since the layer it defines has no guaranteed behaviour. The restriction on the stepsize $h_1$ is derived as in the previous section, while $h_2$ is either positive and satisfies a similar balance law as for the switching in equation \eqref{eq:resLayers}, or it is allowed to be negative. For the dynamical block to be overall contractive, in case of a negative step $h_2$ we constrain it as in equation \eqref{eq:contractiveDyn}, i.e. we impose $|h_2|<2$. In this way, there is not necessarily an alternation of expansive and contractive layers, but the optimiser is free to learn the switching strategy while still guaranteeing the non-expansivity of the dynamical block. For the experiments on CIFAR-100, we do not impose $A$ to be orthogonal, but we normalise it since we have observed improved performance in this way.}

\bl{We compare these alternation strategies with three other networks. The first uses only non-expansive flow maps defined in \eqref{eq:contractiveDyn}, with $\mathrm{ReLU}$ as an activation function. In the experiments, we denote this network as ``non-expansive". We set the weight matrices to be orthogonal and constrain the learnable step sizes to be less than $2$. We then report the results obtained with a more naive way of constraining the Lipschitz constant of a ResNet layer. This approach relies on composing maps of the form
\[
x\mapsto \frac{1}{2}\left(x + A^T\mathrm{ReLU}(Bx+b)\right),\,\,A^TA=B^TB=I,
\]
as suggested in \cite[Appendix D.1]{li2019preventing}. We noticed experimentally that this constraining strategy does not generate very expressive networks, which motivates the research for better 1-Lipschitz ResNet architectures, as proposed in this manuscript. As a general reference, we also include experiments based on a standard ResNet that is not constrained in its weights and is composed of maps of the form
\[
x\mapsto x + A^T\mathrm{ReLU}(Bx+b).
\]
Before commenting on the results, we remark that the na\"ively constrained network and the reference ResNet have double the parameters of the others based on dynamical systems. The rationale for this choice is to compare the networks at the level of the number of computations done per layer instead of based on the parameter count. For this reason, all the networks have the same number of layers. Furthermore, to get sufficiently accurate predictions on clean images, we did not constrain the last linear layer in all the experiments with the na\"ively constrained network and all the CIFAR-100 experiments for the other networks. To jump-start the training of the networks on the CIFAR-100 dataset, we initialised all their layers, but the final projection layer, with the weights obtained on the CIFAR-10 dataset.}

We implement architectures that take as inputs tensors of order three and shape $3\times 32 \times 32$. The first dimensionality of the tensor increases to $32-64-128$ feature maps throughout the network via convolutional layers. For each fixed number of filters, we have four layers of the forms specified above. To be precise, convolutional layers replace the matrix-vector products. 

The network architecture based on~\cref{eq:resLayers} gets close to 90\% test accuracy\red{, on the CIFAR-10 dataset,} when trained with cross-entropy loss and without weight constraints. However, as presented at the beginning of this section, one could consider its Lipschitz constant and its margin at any input point of interest to get robustness guarantees for a network. For this reason, we now focus on constraining the Lipschitz constant of the architecture, and we introduce the loss function we adopt to promote higher margins. As in~\cite{anil2019sorting}, we train the network architecture with the multi-class hinge loss function defined as 
\[
\mathcal{L} = \frac{1}{N}\sum_{i=1}^N \sum\limits_{\substack{j=1 \\ j\neq \ell(x_i)}}^{10}\max\left\{0,\text{margin} - \left(\mathcal{N}(x_i)^Te_{\ell(x_i)} - \mathcal{N}(x_i)^Te_j \right)\right\},
\]
where $\text{margin}$ is a parameter to tune. We train all the networks with this loss function, and with a stochastic gradient descent (SGD) optimiser. Having predictions with higher margins allows us to get more robust architectures if we fix the Lipschitz constant. Still, too high margins can lead to poor accuracy. In the experiments we test the three margin values $0.07$, $0.15$ and $0.3$. For the networks based on dynamical systems, we report the results obtained constraining all the dynamical blocks and the final projection layer. However, we do not constrain the lifting layers. In this way, we can still control the full network's Lipschitz constant, just considering the norms of those lifting layers. On the other hand, we leave some flexibility to the network, which can train better also when we increase the hinge-loss margin. We notice experimentally that the dynamical blocks usually get a small Lipschitz constant. Thus, even when we do not constrain all the layers, the network will still be 1-Lipschitz.\newline\newline 
To get orthogonal convolutional filters, we apply the regularisation strategy proposed in~\cite{wang2020orthogonal}. This strategy is not the only one possible. Still, for this experiment, we preferred it to more involved ones since the main focus has been on the architecture, and the obtained results are satisfactory. Various works, e.g.~\cite{leimkuhler21a,ozay2016optimization,francca2021optimization}, highlight how one can directly constrain the optimisation steps without having to project the weights on the right space or to add regularisation terms. We have not experimented with these kinds of strategies, and we leave them for future study. We also work with an orthogonal initialisation for the convolutional layers. The lifting layers of the networks based on dynamical systems are modelled as $x\mapsto \alpha Wx$ for a convolutional filter $W$ with $\|W\|_2\leq 1$. To constrain the  norm to 1, we add a projection step after the stochastic gradient descent (SGD) method updates the weights, i.e.\ we normalise the weights as $W\mapsto W/\max\{1,\|W\|_2\}$. Here, the $2-$norm of the convolutional filters is computed with the power method as described, for example, in~\cite{meunier2021scalable}. Furthermore, we work with SGD having a learning rate scheduler that divides the learning rate after a fixed number of epochs. Finally, we generate the adversarial examples with the library ``Foolbox" introduced in~\cite{rauber2017foolbox}. We focus on the $\ell^2-$PGD attack and perform ten steps of it. We test different magnitudes of the adversarial perturbations.\newline\newline 
To analyse the results of the experiments, we show how the accuracy of the networks changes as we increase the magnitude of the perturbations and the areas under these curves we get. The Area Under the Curve (AUC) metric is an informative quantity adopted to measure the adversarial robustness \cite{bashivan2021adversarial}. \red{This metric is evaluated by computing the area below the piecewise linear curve obtained by plotting the robust accuracies as in~\cref{fig:plotsRob}. A higher value indicates a better tradeoff between accuracy and robustness.} In~\cref{fig:plotsRob}, we see that the robustness of the constrained neural networks based on dynamical systems improves compared to the baseline ResNet and the na\"ively constrained one. Furthermore, we see that alternating expansive layers in the network improves the tradeoff between clean accuracy and robustness than using a network with only non-expansive layers. To conclude, it is also evident from the experiments that if a more flexible alternating strategy is adopted, the results can improve because while the clean accuracy can increase, the robustness is kept unchanged or improved. \red{In~\cref{fig:plotsAlternation}, we plot the timesteps learned for the networks with a flexible alternation strategy. More precisely, given the 2 consecutive layers defined by
\begin{align*}
\tilde{\Psi}^{h_1}(x) &= x - h_1 A_c^T\Sigma(A_cx+b_c) =: x-h_1 X(A_c,b_c,x),\,\,A_c^TA_c=I \\
\Psi^{h_2}(x) &= x + h_2A^T\mathrm{ReLU}(Ax+b) =: x+ h_2 X(A,b,x),\,\,A^TA=I \nonumber \\
x & \mapsto \tilde{\Psi}^{h_1/2}\circ \tilde{\Psi}^{h_1/2}\circ \Psi^{h_2}(x),
\end{align*}
we plot line segments that are as long as the steps $h_1$ and $h_2$, doing it for all the pairs of such layers. The step $h_2$ can also be negative, leading to non-expansive dynamics. This possibility is the main difference provided by the flexibility of the alternation approach. We notice that, especially for the CIFAR-10 dataset, a timestep is learned to be negative. This is not the case for CIFAR-100. For the case of $\mathrm{margin}=0.15$ and $\mathrm{margin}=0.3$, reported in~\cref{se:expRob}, more steps are negative, especially for the CIFAR-10 experiments. On the other hand, there seems not to be a clear pattern in the step selection. These results suggest the optimiser exploits the freedom introduced due to the flexibility in the step selection and allows getting improved results in some instances. \Cref{se:expRob} collects more details on how the timesteps are constrained.} Furthermore, in~\cref{se:expRob}, we also report the experiments for different margin values.

\red{We remark that the results obtained with our proposed approach are not as good as those provided by the technique of adversarial training yet. On the other hand, our derivations lead to a more efficient training strategy that allows us to get networks with reduced sensitivity without the need to build adversarial examples in the training phase. Additionally, the results in~\cref{fig:plotsRob} show that the proposed constraining strategy allows considerable gains in the accuracy-robustness tradeoff. The proposed framework is general enough to allow for possible improvements and reduce the performance difference with adversarial training. However, how to do so in practice still needs to be understood. We mention some possibilities in \cref{se:conclusion}.}
\begin{figure}[ht!]
\centering
    \centering
    \begin{subfigure}{.49\textwidth}
        \includegraphics[width=\textwidth]{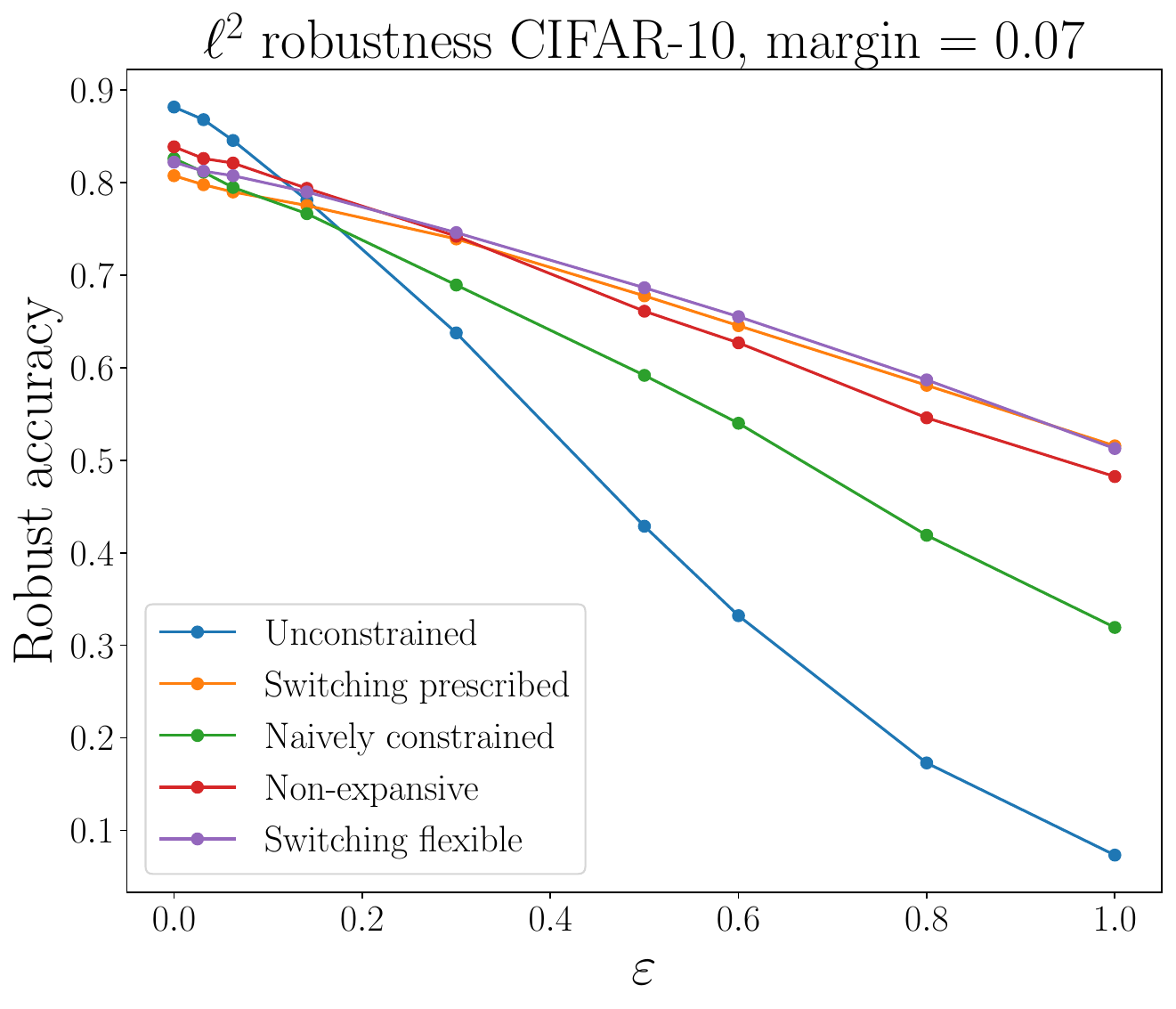}
    \caption{Robustness plot}
    \label{fig:robustness10}
    \end{subfigure}
    \centering
    \begin{subfigure}{.49\textwidth}
        \includegraphics[width=\textwidth]{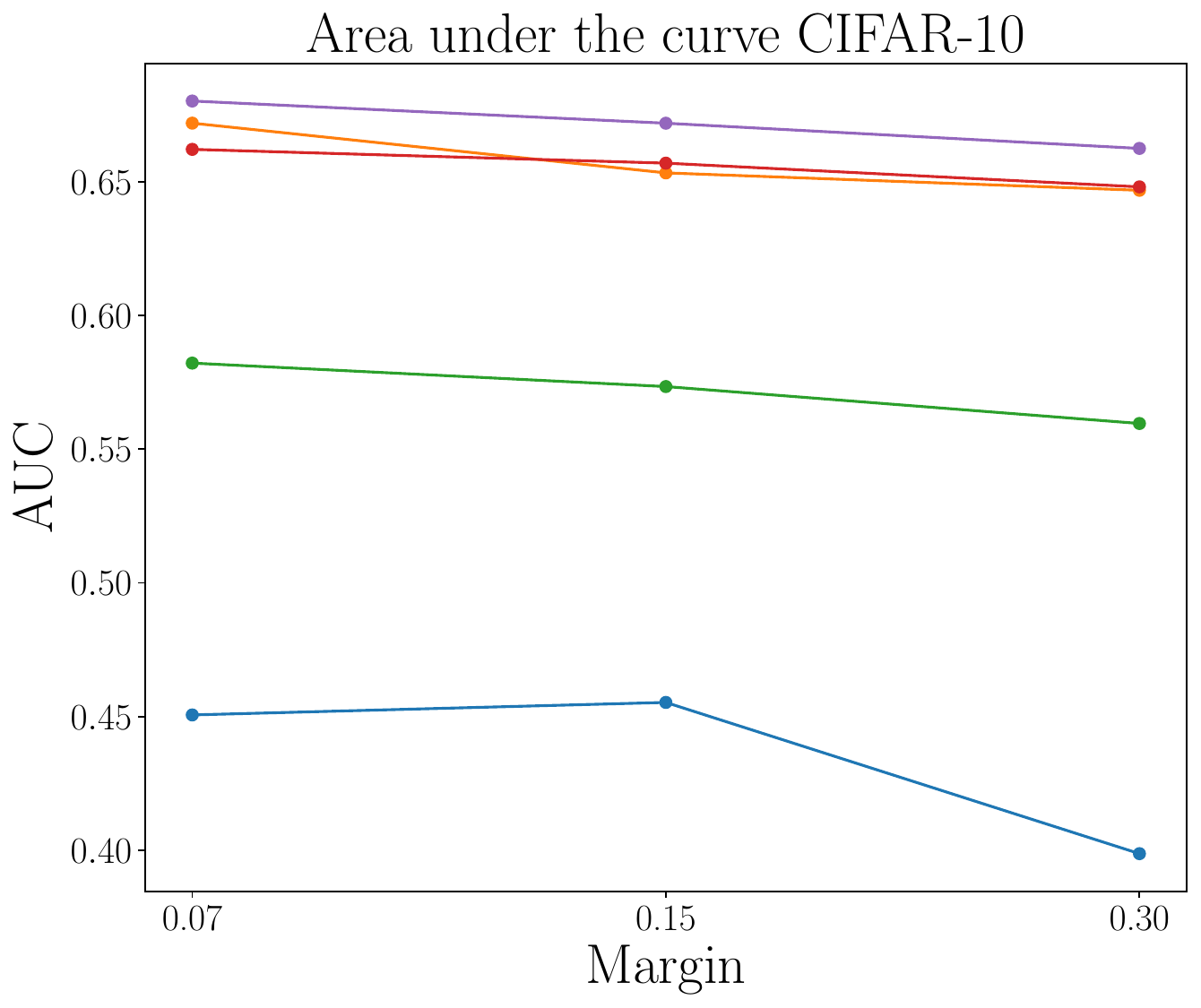}
    \caption{Area under the curve}
    \label{fig:auc10}
    \end{subfigure}

    \centering
    \begin{subfigure}{.49\textwidth}
        \includegraphics[width=\textwidth]{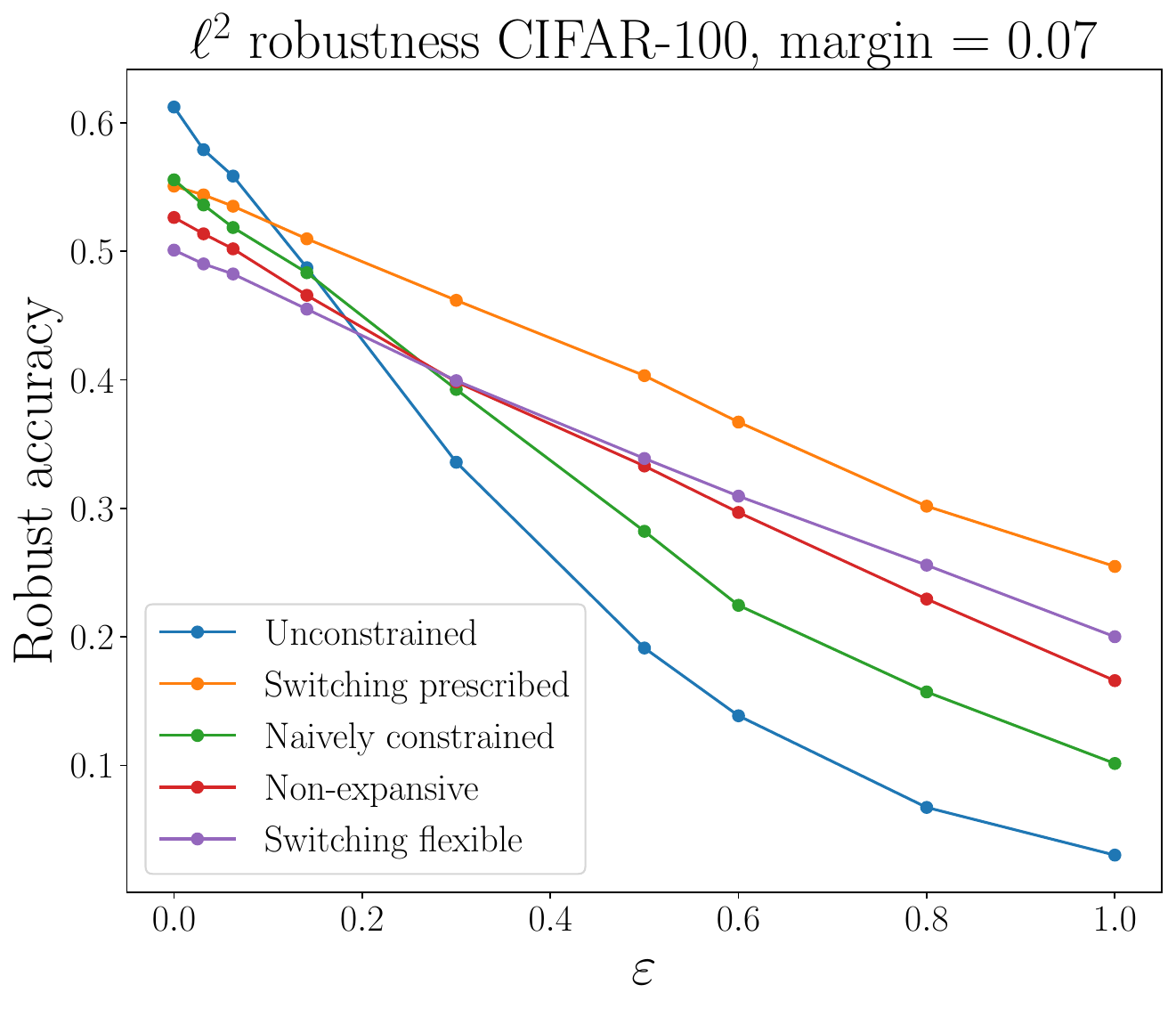}
    \caption{Robustness plot}
    \label{fig:robustness100}
    \end{subfigure}
    \centering
    \begin{subfigure}{.49\textwidth}
        \includegraphics[width=\textwidth]{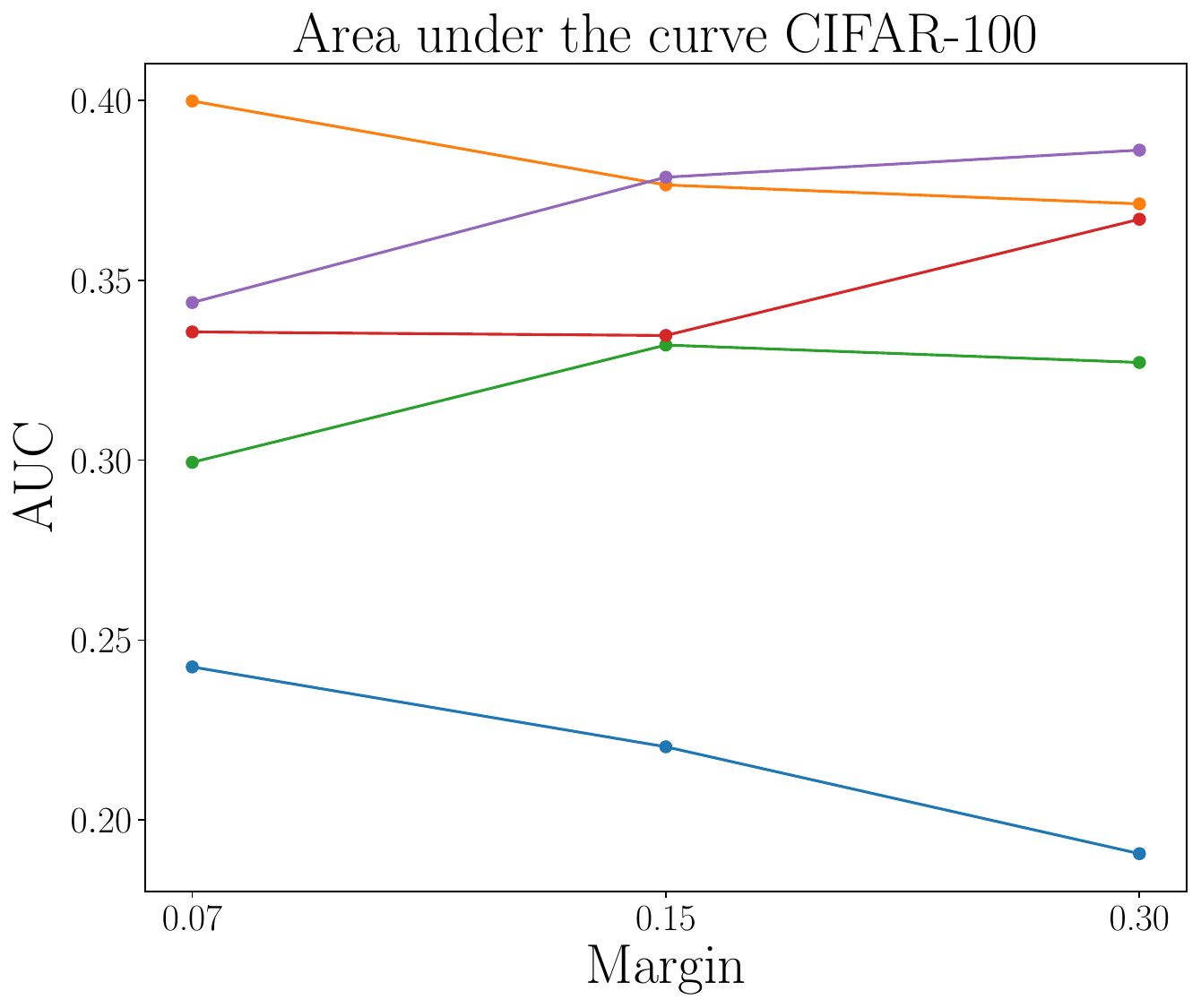}
    \caption{Area under the curve}
    \label{fig:auc100}
    \end{subfigure}
    \caption{On the left: plots of the accuracy against the magnitude of the adversarial PGD perturbation of $1024$ test clean images, comparing various perturbation magnitudes and the 5 networks introduced in the text. On the right: area under the curves, to measure the actual robustness of the models. The legend is shared among the four plots, and for clarity we omit in the plots for the area under the curve.}
    \label{fig:plotsRob}
\end{figure}

\begin{figure}
\centering
    \centering
    \begin{subfigure}{.49\textwidth}
    \includegraphics[width=\textwidth]{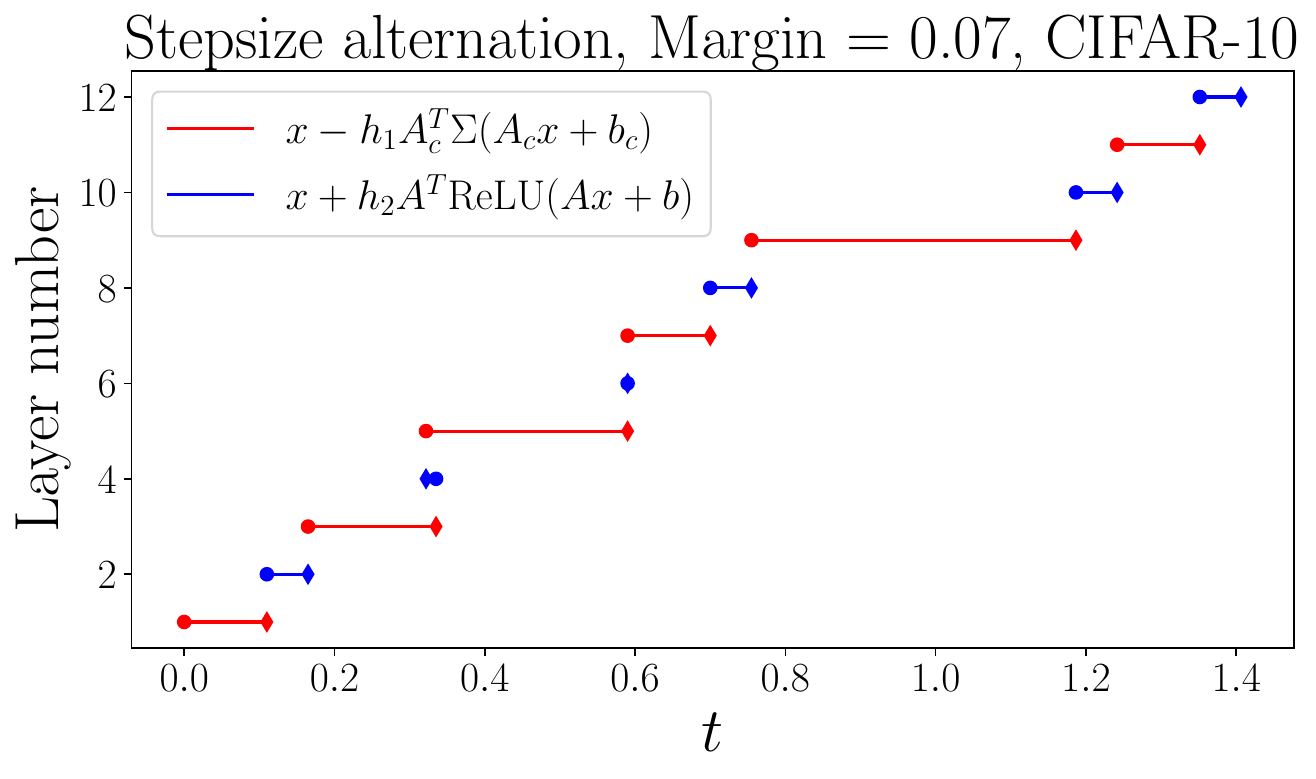}
    \caption{Learned step alternation strategy for CIFAR-10 dataset.}
    \label{fig:alt10}
    \end{subfigure}
    \begin{subfigure}{.49\textwidth}
        \includegraphics[width=\textwidth]{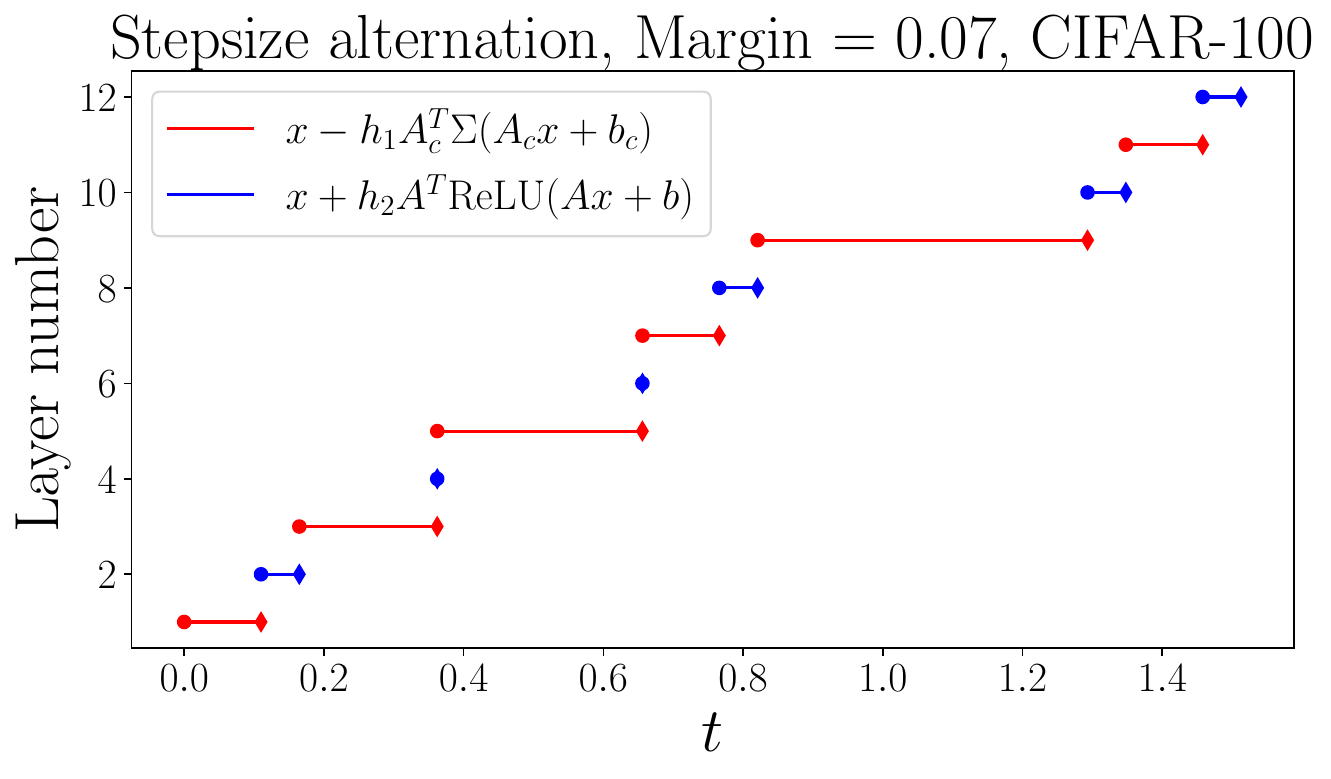}
    \caption{Learned step alternation strategy for CIFAR-100 dataset.}
    \label{fig:alt100}
    \end{subfigure}
    \caption{\red{Representation of the learned stepsizes for the 12 layers characterising the networks giving the results reported in~\cref{fig:plotsRob}. This is the case with $\mathrm{margin}=0.07$, and the other ones are reported in~\cref{se:expRob}. The time instants corresponding to the beginning of the interval where a layer is active are denoted with dots, while the ending time instants with diamonds. This implies that if a diamond on a segment is on the left of a dot, the represented timestep is negative. The abscissa $t$ corresponds to the sum of the timesteps characterising each layer.}}
    \label{fig:plotsAlternation}
\end{figure}

\section{Imposition of other structure}\label{se:structure}
Depending on the problem and the application where a neural network is adopted, the properties that the architecture should satisfy may be very different. We have seen in the previous section a strategy to impose Lipschitz constraints on the network to get some guarantees in terms of adversarial robustness. In that context, the property is of interest because it is possible to see that even when imposing it, we can get sufficiently accurate predictions. Moreover, this strategy allows controlling the network's sensitivity to input perturbations. As mentioned in the introduction, there are at least two other situations where structural constraints might be desirable. The first one is when one knows that the function to approximate has some particular property. The second is when we know that the data we process induces some structure on the function we want to approximate. \newline\newline 
This section supports the claim that combining ODE models with suitable numerical integrators allows us to define structured networks. More precisely, we derive multiple architectures by putting together these two elements. Some of these have already been presented in other works, and others are new. The properties that we investigate are symplecticity, volume preservation and mass preservation. For the first two, we describe how to constrain the dynamical blocks. For the third, we propose how to structure also the linear lifting and projection layers. Moreover, for this latter example, we also report some numerical experiments. \bl{The purpose of the presented toy example is to show that the architecture is computationally realisable and also effective.}
\subsection{Symplectic dynamical blocks}
A function $F:\mathbb{R}^{2n}\rightarrow\mathbb{R}^{2n}$ is said to be symplectic if it satisfies the identity
\[
\frac{\partial F(x)}{\partial x}^T\mathbb{J}\frac{\partial F(x)}{\partial x} = \mathbb{J}\quad \forall x\in\mathbb{R}^{2n},\,\,\mathbb{J}=\begin{bmatrix} 0_n & I_n \\ -I_n & 0_n \end{bmatrix}\in\mathbb{R}^{2n\times 2n}
\] 
with $0_n,I_n\in\mathbb{R}^{n\times n}$ being the zero and the identity matrices. Symplectic maps are particularly important in classical mechanics, because the flow map $\Phi^t$ of a Hamiltonian system $\dot{x}(t) = \mathbb{J}\nabla H(x(t))$ is symplectic, see e.g.~\cite{leimkuhler2004simulating}. This fact implies that if one is interested in approximating such a flow map with a neural network, then structuring it to be symplectic might be desirable. In this direction there are a considerable number of works (see e.g.~\cite{jin2020sympnets,chen2019symplectic,zhu2020deep,burby2020fast}). We mention in particular~\cite{jin2020sympnets} where the authors construct layers of a network to ensure the symplectic property is satisfied. On the other hand, in~\cite{chen2019symplectic} the authors consider a neural network as the Hamiltonian function of a system and approximate its flow map with a symplectic numerical integrator\footnote{We remark that a one-step numerical method $\Psi^h$ is symplectic if and only if, when applied to any Hamiltonian system, it is a symplectic map.}. The simplest symplectic integrator is symplectic Euler, which applied to $H(q,p)=V(q)+K(p)$ computes updates as
\[
q_{n+1} = q_n + h \partial_p K(p_n),\quad p_{n+1} = p_n - h \partial_q V(q_{n+1}).
\]
We now focus on the gradient modules presented in~\cite{jin2020sympnets}, defined by alternating maps of the form
\[
\mathcal{G}_{1}(q,p) = \begin{bmatrix}A^T\diag(\alpha)\Sigma(Ap+a) + q \\ p\end{bmatrix} \quad \mathcal{G}_{2}(q,p) = \begin{bmatrix} q \\ B^T\diag(\beta)\Sigma(Bq+b) + p \end{bmatrix}. 
\]
We notice that we can obtain the same map from a time-switching ODE model. We first introduce the time-dependent Hamiltonian of such a model, which is
\begin{equation}\label{eq:switchHam}
H_{s(t)}(z) = \alpha_{s(t)}^T\Gamma(A_{s(t)}P_{s(t)}z + b_{s(t)}),\quad A_{s(t)}\in\mathbb{R}^{n\times n},\,b_{s(t)}\in\mathbb{R}^{n},
\end{equation}
with $s:[0,+\infty)\rightarrow\mathbb{R}^+$ being piecewise constant. We can suppose without loss of generality that $s(t)\in \{0,1,2,\ldots,K\}$ and that $P_{s(t)}$ alternates between $\Pi_1 = [I_n,0_n]\in\mathbb{R}^{n\times 2n}$ and $\Pi_2=[0_n,I_n]\in\mathbb{R}^{n\times 2n}$. Let now $\Gamma(z) = [\gamma(z_1),\ldots,\gamma(z_n)]$, $\Sigma(z) = [\sigma(z_1),\ldots,\sigma(z_n)]$, and $\gamma'(s)=\sigma(s)$. We then notice that the Hamiltonian vector field associated to $H_{s(t)}$ alternates between the following two vector fields
\[
X_{H_1} = \begin{bmatrix}A_1^T\diag(\alpha_1)\Sigma(A_1p+b_1) & 0\end{bmatrix}^T,\quad X_{H_2} = \begin{bmatrix}0 & -A_2^T\diag(\alpha_2)\Sigma(A_2q+b_2)\end{bmatrix}^T.
\]
We now conclude that if we compute the exact flow of $X_{H_{s(t)}}$ and we take $s(t)$ to be constant on every interval of length 1, we recover the gradient module in~\cite{jin2020sympnets}. \newline\newline 
Similarly, all the network architectures presented in~\cite{chen2019symplectic} and related works are based on defining a neural network $\mathcal{N}(q,p)$ that plays the role of the Hamiltonian function and then applying a symplectic integrator to the Hamiltonian system $\dot{z} = \mathbb{J}\nabla \mathcal{N}(z)$.
The composition of discrete flow maps of the a time independent Hamiltonian gives a symplectic network with shared weights. On the other hand, if the Hamiltonian changes as in~\cref{eq:switchHam}, one gets different weights for different layers and potentially a more expressive model as presented in~\cite{jin2020sympnets}.

\subsection{Volume-preserving dynamical blocks}
Suppose that one is interested in defining efficiently invertible and volume-preserving networks. In that case, an approach based on switching systems and splitting methods can provide a flexible solution. Consider the switching system defined by $\dot{z}(t) = f_{s(t)}(z(t))$, $f_i\in\mathfrak{X}(\mathbb{R}^n)$,
where for every value of $s(t)$, the vector field has a specific partitioning that makes it divergence-free. For example, if we have $n=2m$,
\[
f_{s(t)}(z) = \begin{bmatrix} u_{s(t)}(z[m:]) & v_{s(t)}(z[:m]) \end{bmatrix}^T
\]
satisfies such a condition and its flow map will be volume-preserving. We can numerically integrate such a vector field while preserving this property. Indeed, we can apply a splitting method based on composing the exact flow maps of the two volume-preserving vector fields
\[
f_{s(t)}^1(z)=\begin{bmatrix} u_{s(t)}(z[m:]) & 0 \end{bmatrix}^T,\quad
f_{s(t)}^2(z)=\begin{bmatrix} 0 & v_{s(t)}(z[:m])\end{bmatrix}^T.
\]
This approach gives architectures that are close to the ones of RevNets (see  e.g.~\cite{gomez2017reversible}). The inverse of the network is efficient to compute in this case, and this translates into memory efficiency since one does not need to save intermediate activation values for the backpropagation. A particular class of these blocks can be obtained with second-order vector fields and, in particular, with second-order conservative vector fields (hence Hamiltonian):
 $\ddot{x}(t) = f_{s(t)}(x(t)),\text{ or }\ddot{x}(t) = -\nabla V_{s(t)}(x(t))$, where, for example, $V_{s(t)}(x) = \boldsymbol{\alpha}^T\Gamma(A_{s(t)}x+b_{s(t)})$ for some $\Gamma = [\gamma, \ldots, \gamma]$. With a similar strategy, one can also derive the volume-preserving neural networks presented in~\cite{bajars2021locally}.

\subsection{Mass-preserving neural networks}
The final property we focus on is mass preservation. By mass preservation, we refer to the conservation of the sum of the components of a vector (see~\cite{blanes2021positivity}). \bl{This property is typical of semi-discretisations of mass-preserving PDEs, models for chemical reactions, for population dynamics and ecology (see e.g. \cite{refMass1,refMass2,refMass3}).} More explicitly, one could be interested in imposing such a structure if the goal is to approximate a function $F:\mathbb{R}^n\rightarrow\mathbb{R}^m$ that is known to satisfy $T_nx = \sum_{i=1}^{n}x_i=T_mF(x)=\sum_{i=1}^mF(x)_i$. A simple way to impose such property is by approximating the target function $F:\mathbb{R}^n\rightarrow\mathbb{R}^m$ as
\[
F(x)\approx \frac{\sum_{i=1}^n x_i}{\sum_{j=1}^m \tilde{F}(x)_j} \tilde{F}(x) 
\]
where $\tilde{F}:\mathbb{R}^n\rightarrow\mathbb{R}^m$ is any sufficiently expressive neural network. However, this choice might lead to hard training procedures because of the denominator. Imposing this structure at the level of network layers is not so intuitive in general. Hence, we rely again on a suitable ODE formulation. A vector field $X\in\mathfrak{X}(\mathbb{R}^n)$ whose flow map preserves the sum of the components of the state vector is simply one having a linear first integral $g(x) = \boldsymbol{1}^Tx = \sum_{i=1}^n x_i$. Thus, we can design vector fields of the form
\begin{equation}\label{eq:skew}
\dot{y}(t) = (A(y)-A(y)^T)\boldsymbol{1},\,\, A:\mathbb{R}^n\rightarrow \mathbb{R}^{n\times n},
\end{equation}
and this property will be then a natural consequence of the exact flow map. This mass conservation could also be extended to a weighted-mass conservation, and we would just have to replace $\boldsymbol{1}$ with a vector of weights $\boldsymbol{\alpha}$. This extension does not, however, allow to change the dimensionality from a layer to the next one as easily. To model these vector fields, we can work with parametric functions like $\tilde{f}(x) = B^T\Sigma(Ax+b)\in\mathbb{R}^{n(n-1)/2}$ and use them to build the upper triangular matrix-valued function $A$ in~\cref{eq:skew}. As presented in \cite[Chapter 4]{geomBook}, it is also immediate to impose this property at a discrete level since every Runge-Kutta or multistep method preserves linear first integrals without time-step restrictions. Thus, a possible strategy to model mass-preserving neural networks is based on combining layers of the following types:
\begin{enumerate}
\item Lifting layers: $L:\mathbb{R}^k\rightarrow\mathbb{R}^{k+s}$, $L(x_1,\ldots,x_k) = (x_1,\ldots,x_k,0,0,\ldots,0)$,
\item Projection layers: $P:\mathbb{R}^{k+s}\rightarrow\mathbb{R}^k$, $P(x_1,\ldots,x_k,x_{k+1},\ldots,x_{k+s}) = (x_1+o,\ldots,x_{k}+o)$, with $o=\sum_{i=1}^s x_{k+i}/s,$
\item Dynamical blocks: one-step explicit Euler discretisations of~\cref{eq:skew}.
\end{enumerate}
To test the neural network architecture, we focus on the approximation of the flow map of the SIR-model
\begin{equation}
\dot{y} = \begin{bmatrix} -y_1y_2 & y_1y_2-y_2 & y_2 \end{bmatrix}^T = X(y)^T.
\label{eq:SIR}
\end{equation}
\bl{This experiment relates to the research area of data-driven modelling, which has attracted a high amount of interest in recent years, especially through the tools provided by machine learning (see, e.g., \cite{celledoni2022learning,eidnes2022order,bajars2021locally,chen2019symplectic})}. We model the neural network as discussed above. We approximate the $1-$flow map of~\cref{eq:SIR} working with pairs of the form $\{(x_i,y_i=\Psi^1_X(x_i))\}_{i=1,\ldots,N}$\footnote{Here with $\Psi^1_X(x_i)$ we refer to an accurate approximation of the time-1 flow map of X applied to $x_i$}. \bl{In this context we suppose it is not possible to integrate in time the system of ODEs because this is not available, and what is provided is just a set of observed trajectories.}  The plots in~\cref{fig:massPres} represent the first two components of the solution for the SIR model. All the line segments connect the components of the initial conditions with those of the time-1 updates. The considerable benefit of mass-preservation as a constraint is that it allows interpretable outputs. Indeed, in this case the components of $y$ represent the percentages of three species on the total population, and the network we train still allows to get this interpretation being mass-preserving.
\begin{figure}[ht!]
\centering
\begin{subfigure}{0.49\textwidth}
\centering
\includegraphics[width=\textwidth]{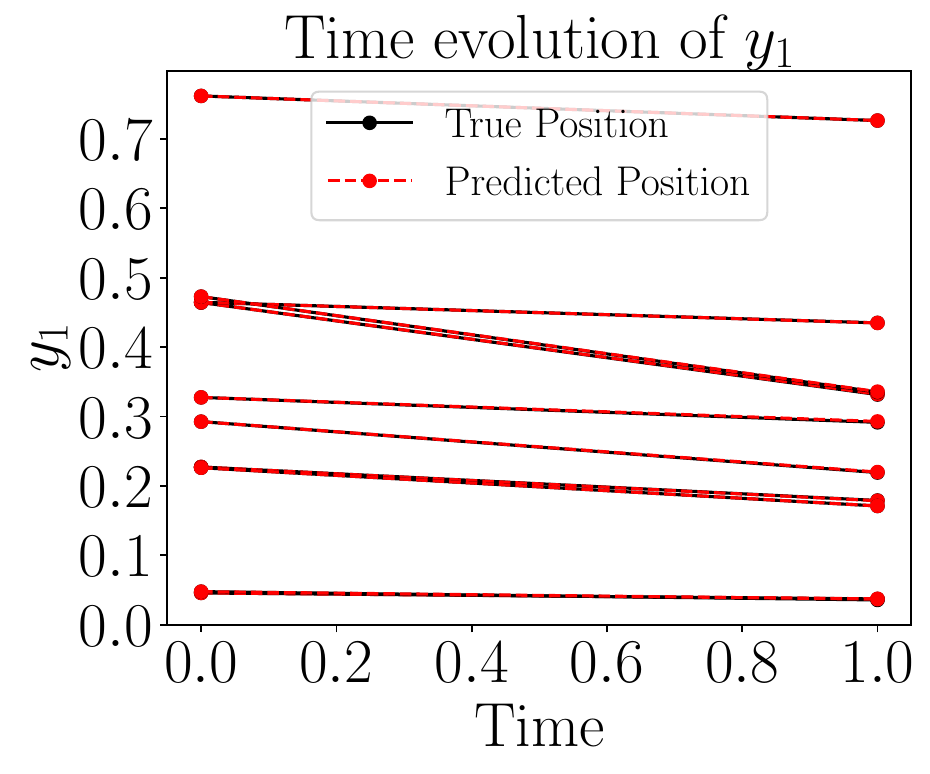}
\end{subfigure}
\begin{subfigure}{0.49\textwidth}
\centering
\includegraphics[width=\textwidth]{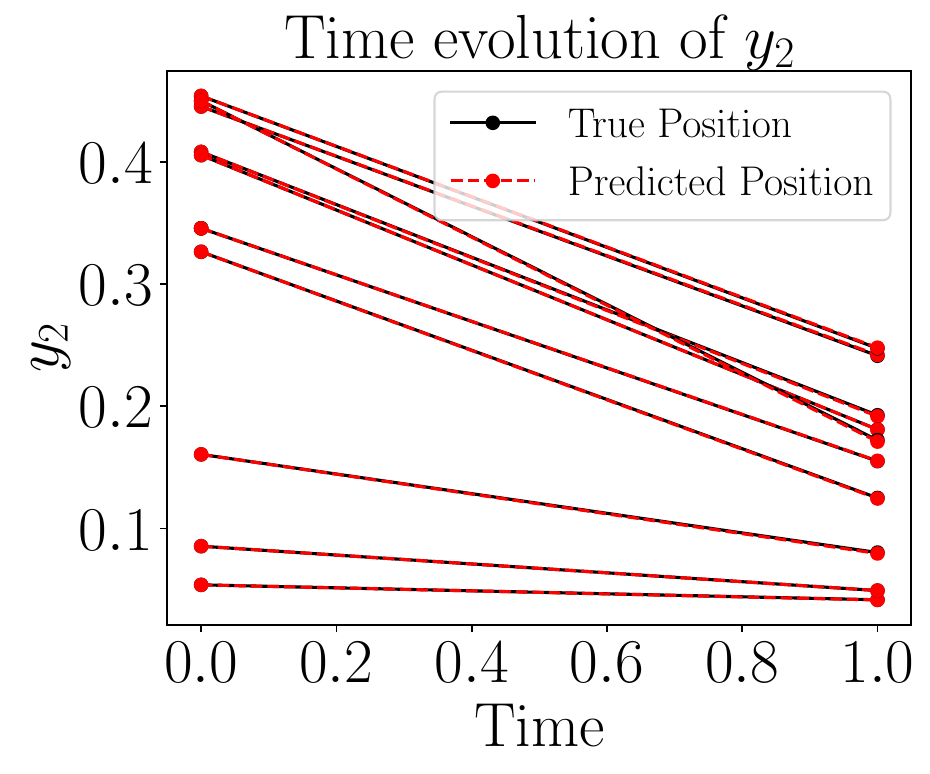}
\end{subfigure}
\caption{Plots of the approximation of the time $1-$flow map of the SIR model~\cref{eq:SIR} for 10 test initial conditions. We report the first two components of the solutions. Each point represents either an initial condition (at time 0), or a time-1 update.}
\label{fig:massPres}
\end{figure}

\section{Conclusion and future directions}\label{se:conclusion}
In this work, we have introduced a framework to combine the design of ODE models with the choice of proper numerical methods to, in turn, obtain neural networks with prescribed properties. After introducing and motivating the approach, we proved two universal approximation theorems. The first one relates to sphere-preserving and gradient vector fields, while the second one involves Hamiltonian vector fields. We then obtained Lipschitz-constrained ResNets, focusing mainly on how to introduce layers that are not 1-Lipschitz. We then applied this construction to get neural networks with adversarial robustness guarantees. Finally, to show the framework's flexibility, we demonstrated how to design dynamical blocks that are symplectic, volume-preserving and mass-preserving.

\red{The main application investigated in this manuscript is the one of adversarial robustness. Our experiments highlight that the robustness of neural networks to input perturbations can be improved using structured neural networks. However, the obtained results are competitive with other constraining strategies but not with adversarial training, which still provides state-of-the-art performance. We plan to optimise the proposed approach to get higher clean accuracy, possibly by designing a better optimisation strategy or designing other more expressive families of expansive and contractive vector fields}.

Throughout the manuscript we have focused on explicit numerical methods as tools to generate neural network architectures. However, many geometric integrators are implicit (see e.g.~\cite{geomBook}); thus, this remains a direction to pursue in further work so that the framework can be extended to other properties (see e.g.~\cite{bai2019deep,look2020differentiable}).

We have adopted the formalism of piecewise-autonomous dynamical systems to design neural networks without heavily relying on the theory of time-switching systems. However, switching systems are a well-studied research area (see e.g.~\cite{liberzon2003switching,liberzon1999basic,alpcan2010stability}), and it seems natural to study them further and their use to design neural network architectures. 

Finally, we remark that imposing properties on neural networks is a promising strategy to make them more understandable, interpretable, and reliable. On the other hand, it is also clear that constraining the architecture can considerably decrease the network's expressivity in some cases. Thus, it remains to understand when it is preferable to replace hard constraints with soft constraints promoting such properties without imposing them by construction. 

{\bf Acknowledgements} The authors would like to thank the Isaac Newton Institute for Mathematical Sciences, Cambridge, for support and hospitality during the programmes {\it Mathematics of deep learning\/} and {\it Geometry, compatibility and structure preservation in computational differential equations\/} where work on this paper was significantly advanced. This work was supported by EPSRC grant no EP/R014604/1. EC and BO have received funding from the European Union’s Horizon 2020 research and innovation programme under the Marie Sk{\l}odowska-Curie grant agreement No 860124.

\bibliographystyle{siamplain}
\bibliography{biblio.bib}

\appendix

\section{Some numerical experiments for data-driven modelling and regression}\label{se:expReg}

We now apply the theoretical background introduced in~\cref{se:approx} to two approximation tasks. More precisely, we verify whether the introduced architectures are complicated to train in practice or whether they can achieve good performances. The two tasks of interest are the approximation of a continuous scalar function and the approximation of a $\mathcal{C}^1$ vector field starting from a set of training trajectories. \newline\newline 
We start with the approximation of $f(x) = x^2 + |x| + \sin{(x)}$, $x\in\mathbb{R}$, and $g(x,y)=\sqrt{x^2+y^2}$. The goal here is to approximate them by composing flow maps of vector fields that are structured as 
\[
X_G(z) = A^T \diag(\boldsymbol{\alpha})\Sigma(Az+b) = \nabla \left(\boldsymbol{\alpha}^T\Gamma(Az+b)\right),
\]
\[
X_S(z) = (A(z)-A(z)^T)z.
\]
More precisely, following the result in~\cref{thm:switch}, we compose the flow maps of $X_G$ and $X_S$, maintaining the same time step for pairs of such flow maps. We report the results in~\cref{fig:regression}.

\begin{figure}[ht!]
\begin{subfigure}{.45\textwidth}
\centering
\includegraphics[width=\textwidth]{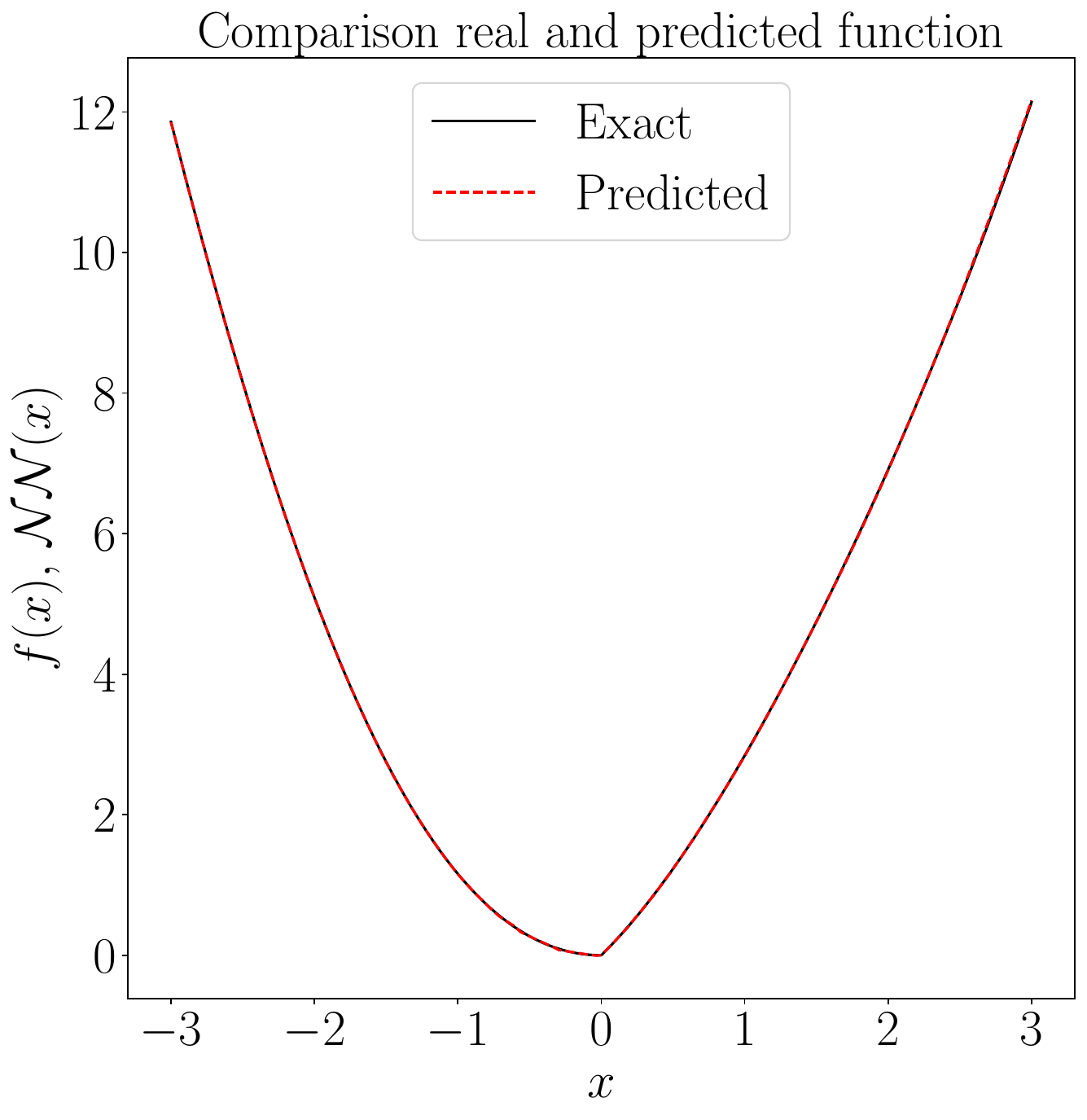}
\caption{Approximation of $f(x)$}
\end{subfigure}
\begin{subfigure}{.45\textwidth}
\centering
\includegraphics[,width=\textwidth]{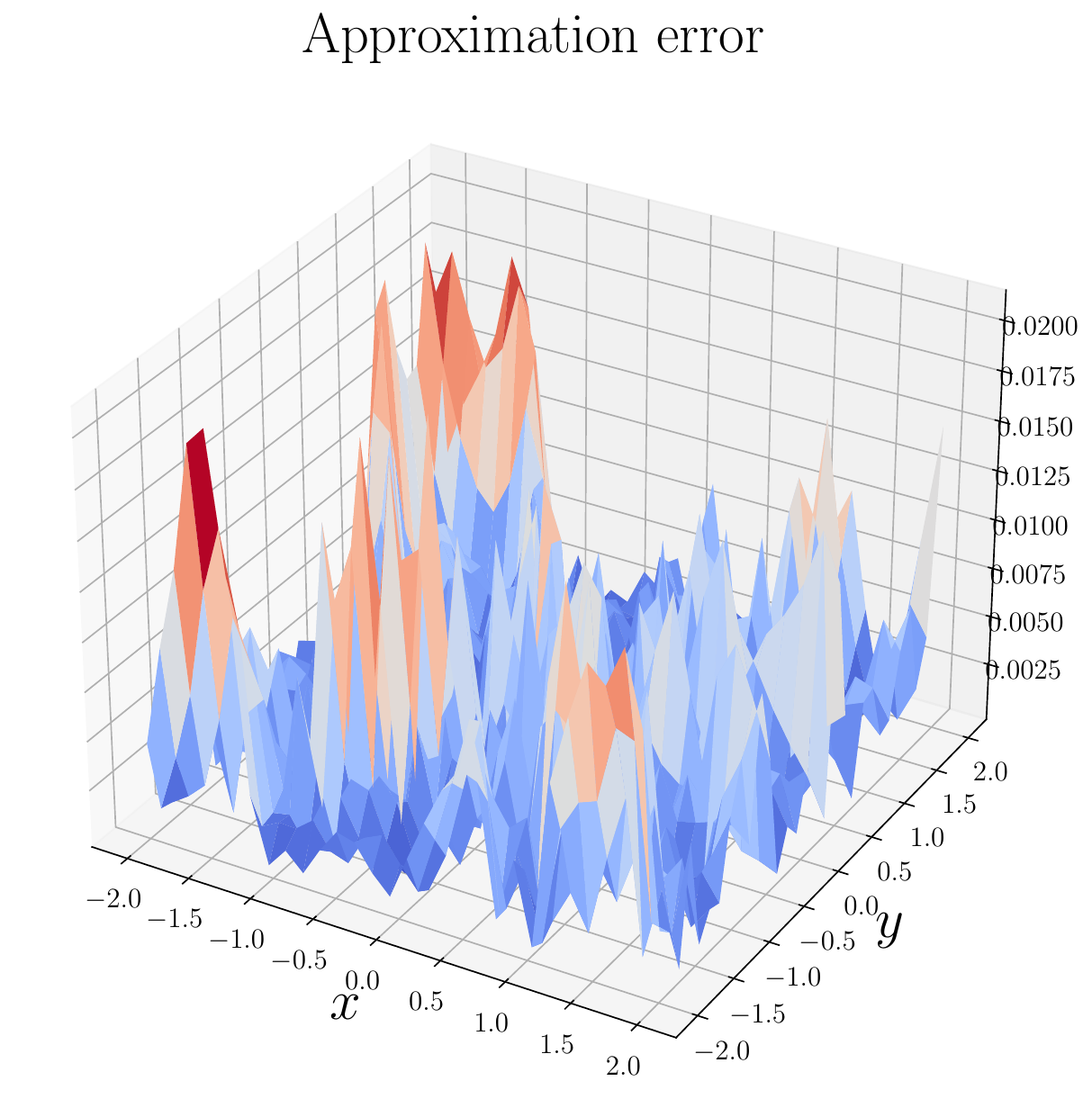}
\caption{Approximation of $g(x,y)$}
\end{subfigure}
\caption{Comparison between the trained networks and the true functions.}
\label{fig:regression}
\end{figure}
The second experiment that we report, is the one of approximating a vector field $X\in\mathfrak{X}(\mathbb{R}^n)$ starting from a set of training pairs $\{(x_i,y_i)\}_{i=1,\ldots,N}$, with $y_i = \Psi^h_X(x_i)$, for an accurate approximation $\Psi^h_X$ of the time-$h$ flow of $X$. We recall that the universal approximation result based on the Presnov decomposition allows approximating any vector field $X$ as 
\[
X_{\theta}(z)=A^T\diag(\alpha)\Sigma(Az+b)+(B(z)-B(z)^T)z= X_G(z) + X_S(z),
\]
as long as the weights are chosen correctly. In principle, calling $\Psi^h$ any numerical method applied to $X_{\theta}$ and said $\hat{y}_i = \Psi^h(x_i)$, one could train the weights of $X_{\theta}$ so that they minimise
\[
\mathcal{L} = \frac{1}{N}\sum_{i=1}^N \left\|y_i - \hat{y}_i\right\|^2.
\]
However, because of the properties of $X_G$ and $X_S$, we choose to preserve them with $\Psi^h$ and apply a splitting method. In other words, we apply to $X_S$ and $X_G$ two integrators $\Psi_S^h$ and $\Psi_G^h$, then compose them to obtain $\Psi^h=\Psi_G^h\circ \Psi_S^h$. We choose $\Psi_S^h$ to be an explicit method that preserves the conserved quantity $\|x\|^2$, while $\Psi_G^h$ to be a discrete gradient method (see  e.g.~\cite{eidnes2022order,mclachlan1999geometric}) so that it preserves the dissipative nature of $X_G$. As said before, this splitting strategy is not necessary in principle. However, we propose it as an alternative inspired by all the works on Hamiltonian neural networks (see e.g.~\cite{chen2019symplectic,celledoni2022learning,eidnes2022order}) where geometric integrators are often utilised. Furthermore, it would be interesting to understand if this or other splitting strategies give better approximation results or theoretical guarantees, but this goes beyond the scope of this work.\newline\newline 
We choose $\Psi_S^h$ as a modified Euler-Heun method, following the derivation presented in~\cite{calvo2006preservation}, so that it is explicit and it also preserves $\|x\|^2$. For $\Psi_G^h$ we use instead the Gonzalez discrete gradient method (see e.g.~\cite{mclachlan1999geometric}). 
\begin{figure}[ht!]
\begin{subfigure}{.49\textwidth}
\centering
\includegraphics[width=.8\textwidth]{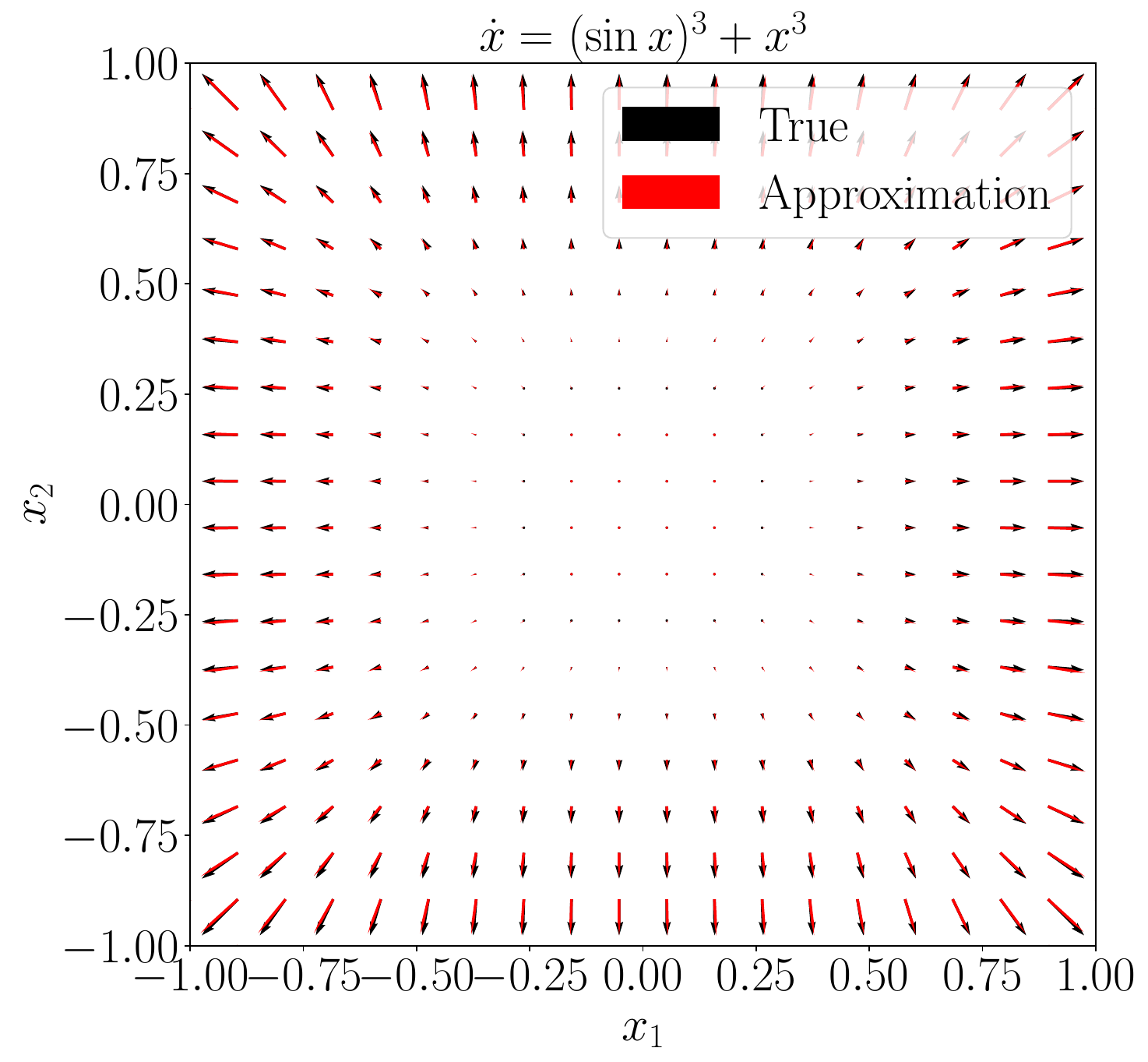}
\end{subfigure}
\begin{subfigure}{.49\textwidth}
\centering
\includegraphics[width=.8\textwidth]{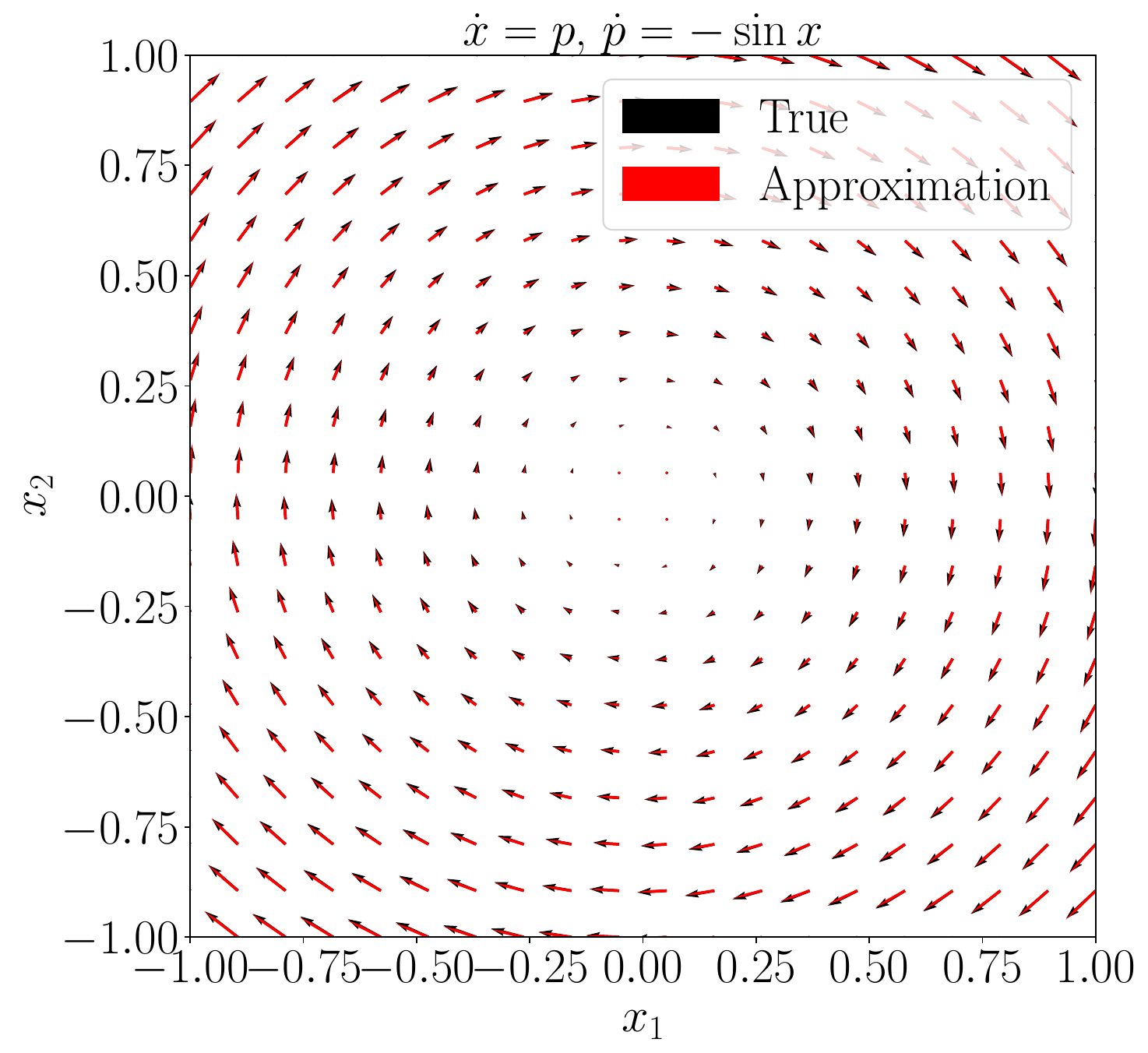}
\end{subfigure}
\caption{In the two plots, we compare the true and predicted vector fields.}
\label{fig:phasePort}
\end{figure}
We remark that discrete gradient methods applied to $\dot{x}(t) = -\nabla V(x(t))$ are of the form
\[
x^{n+1} = x^n - h\overline{\nabla}V(x^n,x^{n+1})
\]
and they are hence implicit. However, since we have the trajectories available, i.e.\ the $y_i$s are known, we do not have to solve a non-linear system of equations. Indeed, the problem of approximating $X$ amounts to minimise the following cost function
\[
\mathcal{L} = \frac{1}{N} \sum_{i=1}^N\left\|y_i - \left(\Psi_S^h(x_i) - h \overline{\nabla}V(\Psi_S^h(x_i),y_i)\right)\right\|^2.
\]
In~\cref{fig:phasePort}, we report the results obtained for the following two vector fields
\[
X_1(x) = \begin{bmatrix} (\sin{x_1})^3+x_1^3\\ (\sin{x_2})^3+x_2^3 \\ (\sin{x_3})^3+x_3^3 \\ (\sin{x_4})^3+x_4^3 \end{bmatrix} \quad X_2(x) = \begin{bmatrix} x_2 \\ -\sin{x_1} \end{bmatrix}.
\]
\section{Non-expansive networks with non-Euclidean metric}\label{se:generalmetric}

\begin{figure}[ht!]
    \centering
    \includegraphics[width=.7\textwidth]{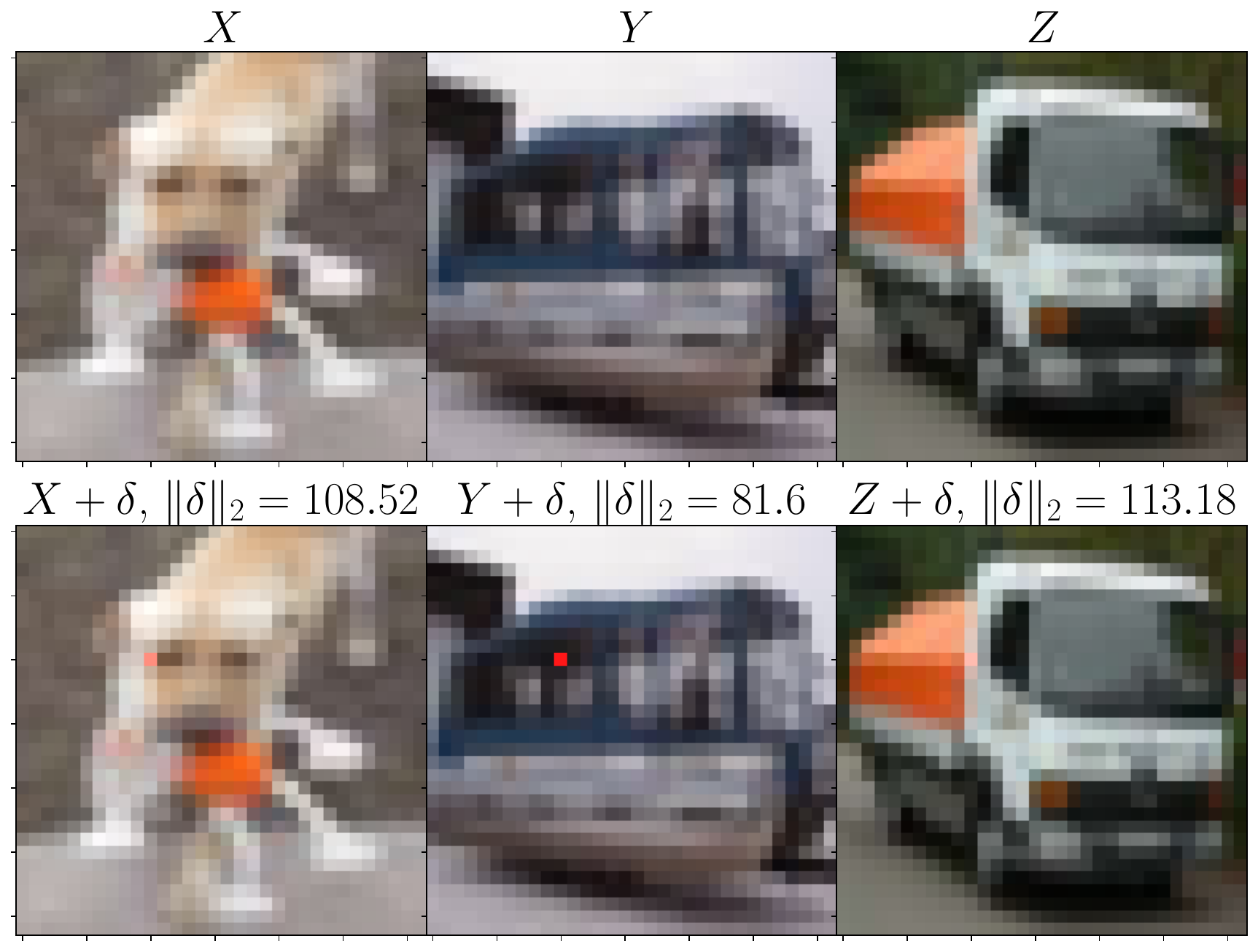}
    \caption{Three clean images from the CIFAR-10 dataset at the top and a random perturbation of the red channel in one pixel at the bottom. For humans the images of the two rows clearly associate to the same object, even with this pixel perturbation.}
    \label{fig:perturb}
\end{figure}
In~\cref{se:lipschitz} we have introduced a way to generate networks that are non-expansive for the Euclidean metric on the input space. We now propose a strategy to generalise the reasoning. It is intuitive, and also evident from~\cref{fig:perturb}, that $\ell^p$ norms of $\mathbb{R}^n$ are not the ones humans utilise to compare pictures. There have been many attempts to design similarity measures between images (see  e.g.~\cite{wang2005euclidean, simard1992efficient}), but it is still not evident what should be the preferred choice. To make the model~\cref{eq:alternation} introduced in~\cref{se:lipschitz} more general, we hence show how it can be made contractive for a more generic metric $d$ defined by a symmetric and positive definite matrix $M\in\mathbb{R}^{n\times n}$. We introduce the notation $\langle v,w\rangle_M = v^TMw$ for any pair of vectors $v,w\in\mathbb{R}^n$. Let again $\Sigma(z)=[\sigma(z_1),\ldots,\sigma(z_n)]$ where $\sigma$ is an increasing scalar function. We focus on the autonomous dynamical system
\begin{equation}
\dot{z}(t) = -W^T\Sigma(MWz+b),\quad W \in\mathbb{R}^{n\times n},\,b\in\mathbb{R}^n,
\label{eq:generalM}
\end{equation}
that is no more a gradient vector field, but it has similar properties to the one studied above. We suppose $M$ is constant and the same for the other involved weights. How this reasoning extends to time-switching systems, as discussed throughout the paper, is quite natural and follows the procedure seen for $M$ being the identity matrix, compare in particular~\cref{se:lipschitz}. We now verify the contractivity of the ODE~\ref{eq:generalM} with respect to the metric defined by $M$:
\begin{align*}
\frac{\ddiff}{\ddiff t}\frac{1}{2}\|z(t)-y(t)\|^2_M &= \frac{1}{2}\frac{\ddiff}{\ddiff t}\left((z(t)-y(t))^TM(z(t)-y(t))\right)\\
&= -\langle W^T\Sigma(MWz+b)-W^T\Sigma(MWy+b),z-y\rangle_M \\
&= -\langle \Sigma(MWz+b)-\Sigma(MWy+b),MWz-MWy\rangle \leq 0.
\end{align*}
This result implies that all the trajectories of~\cref{eq:generalM} will converge to a reference trajectory if the convergence is measured using the metric defined by $M$. Notice that the scalar product in the last line is the canonical one of $\mathbb{R}^n$. Hence, if we have that $\gamma$ is strongly convex, we can still combine these dynamics with expansive vector fields.

\section{Additional details on adversarial robustness}\label{se:expRob}
\begin{figure}[ht!]
\begin{subfigure}{0.49\textwidth}
\centering
\includegraphics[width=\textwidth]{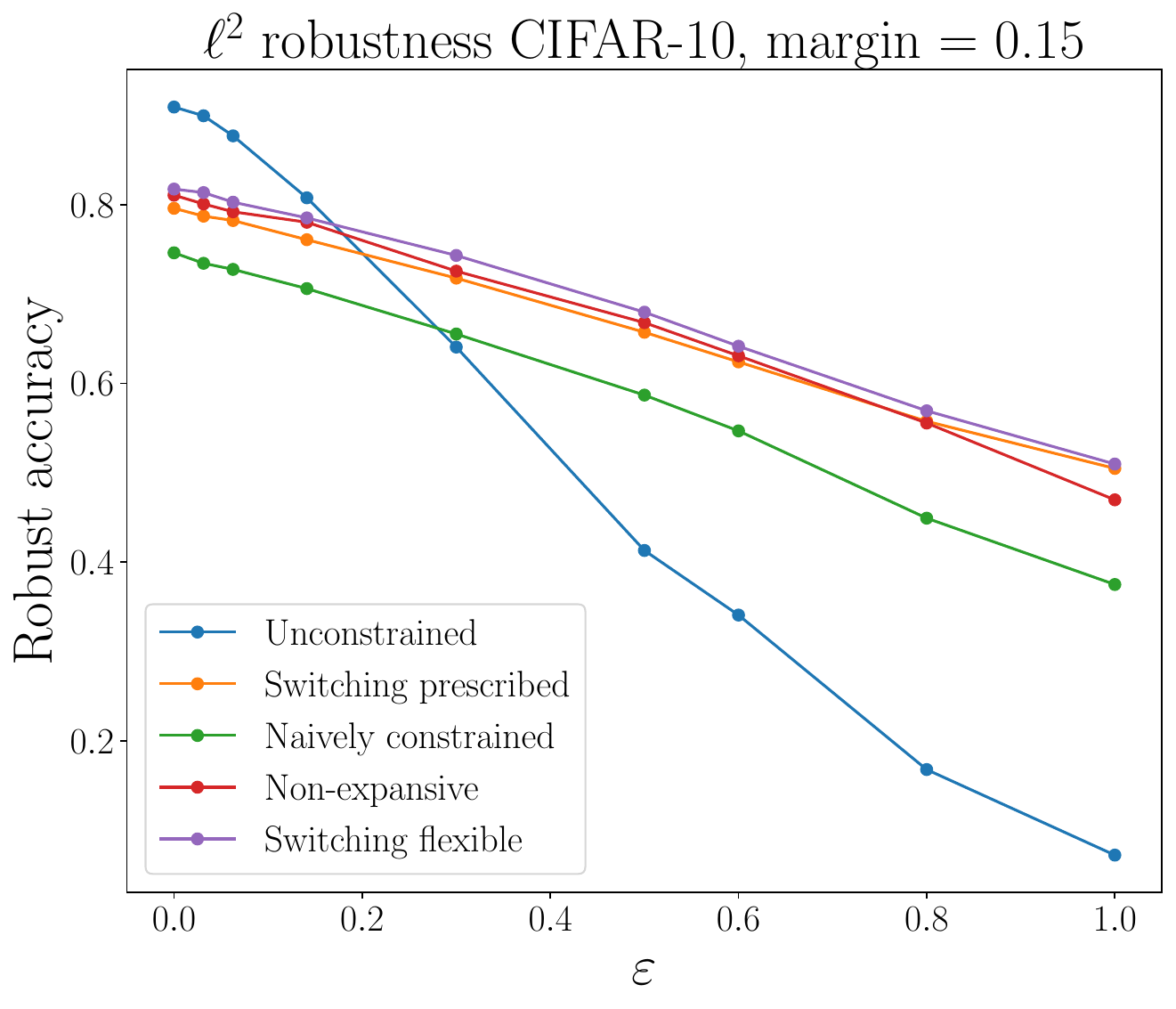}
\end{subfigure}
\begin{subfigure}{0.49\textwidth}
\centering
\includegraphics[width=\textwidth]{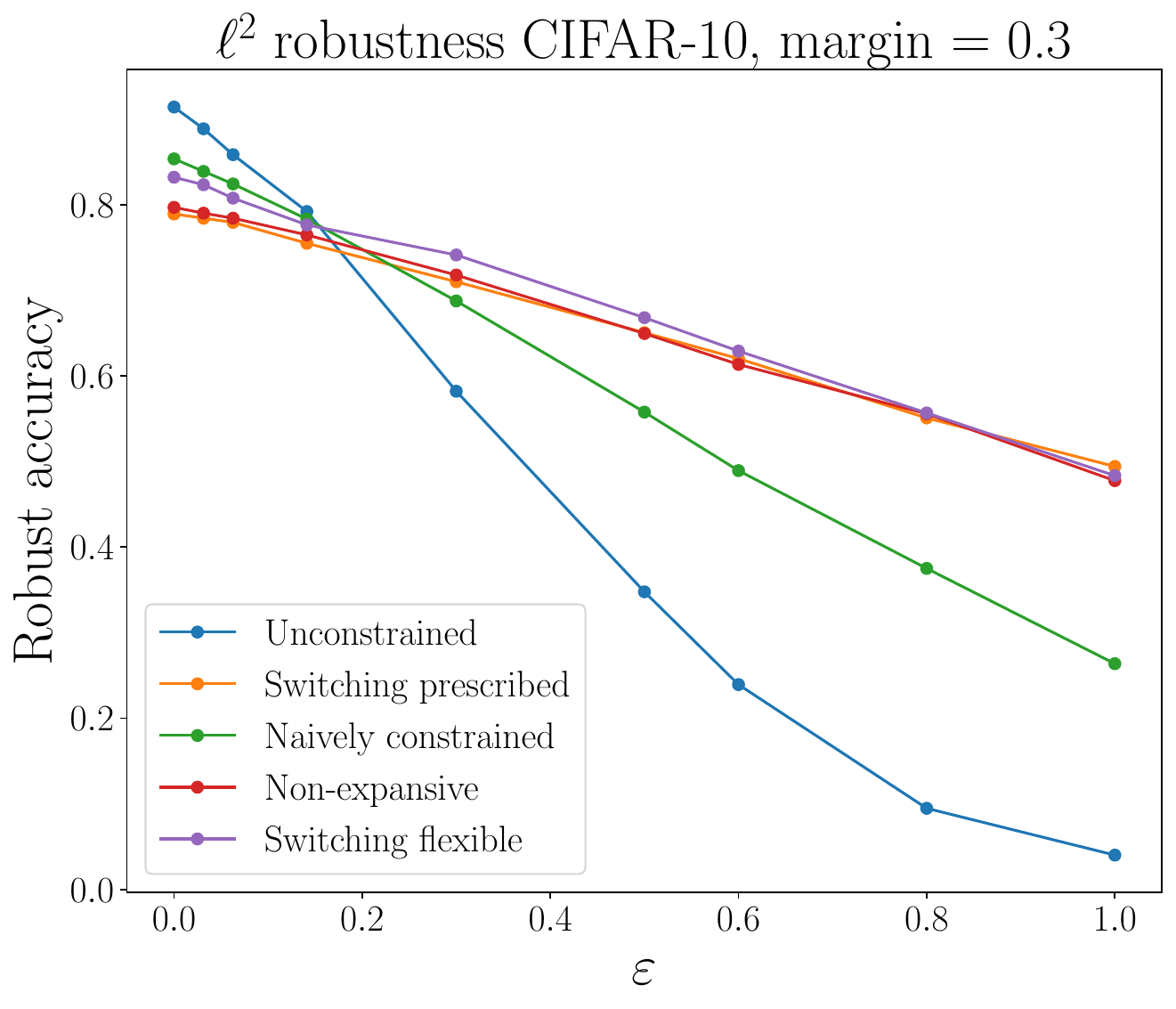}
\end{subfigure}

\begin{subfigure}{0.49\textwidth}
\centering
\includegraphics[width=\textwidth]{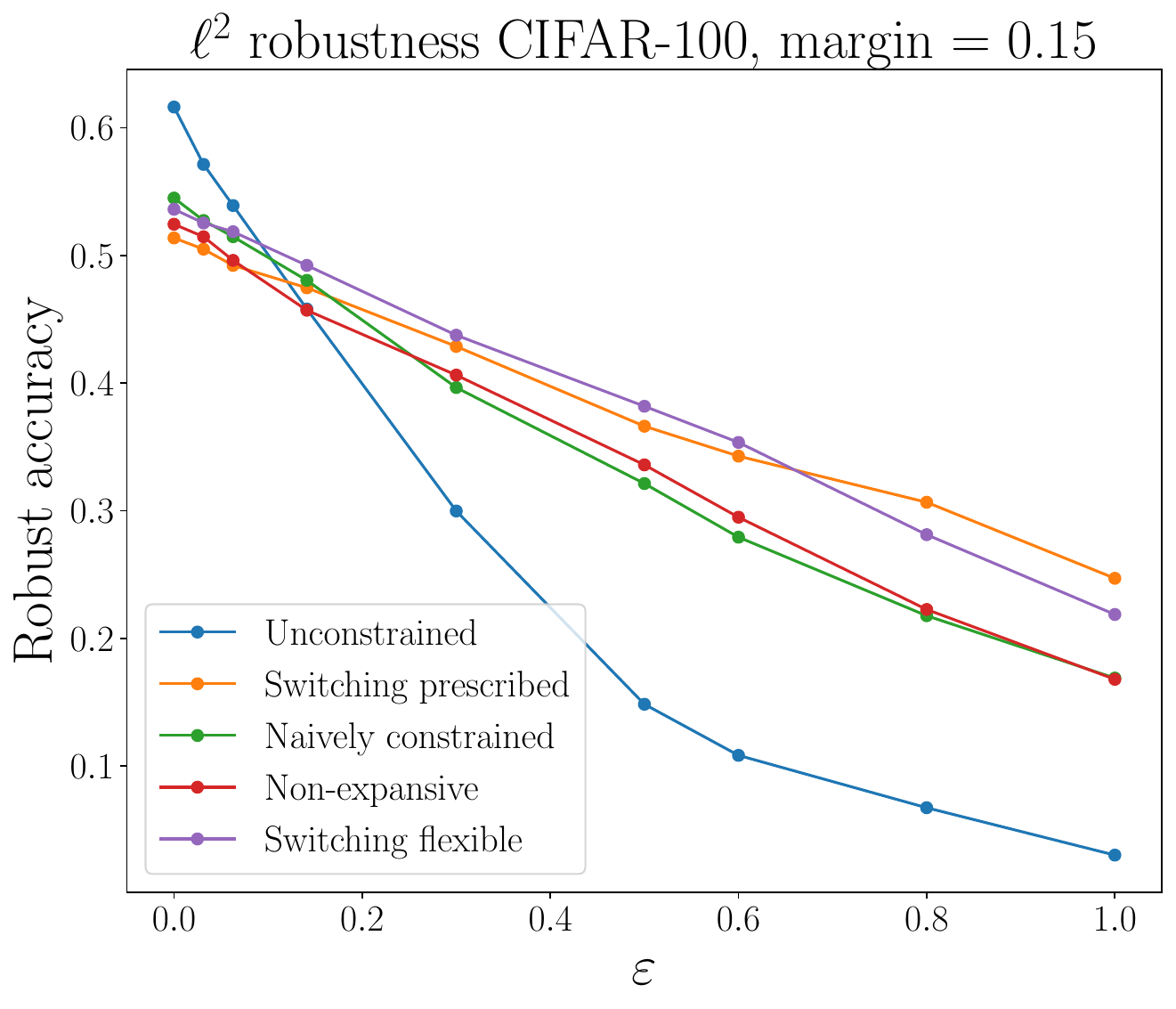}
\end{subfigure}
\begin{subfigure}{0.49\textwidth}
\centering
\includegraphics[width=\textwidth]{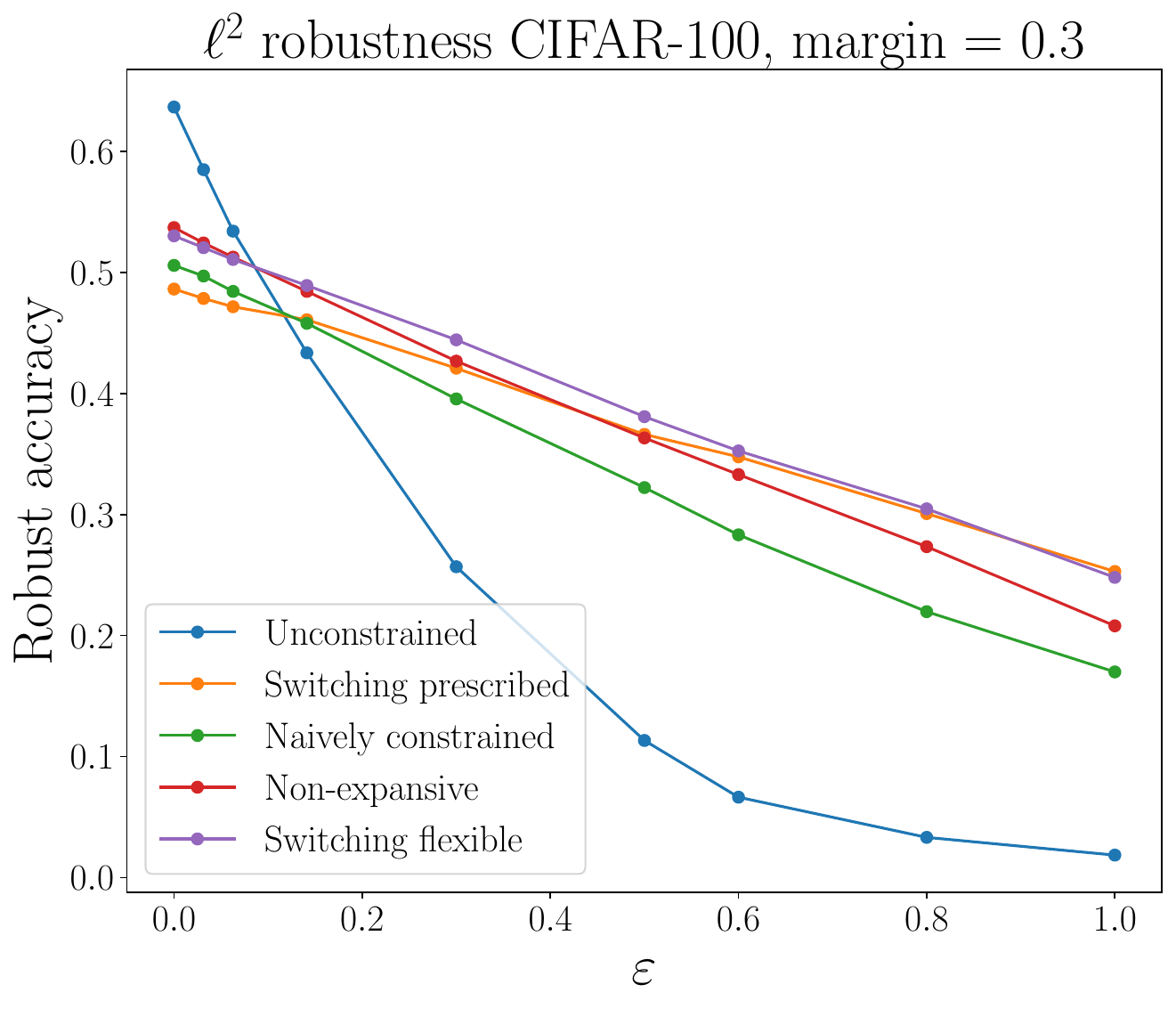}
\end{subfigure}
\caption{Behaviour of the test accuracy on 1024 images, under perturbations of different magnitude $\varepsilon$. On the left we report the experiments with margin value $0.15$ and on the right with $0.3$.}
\label{fig:marginTestRob}
\end{figure}
This section presents additional plots related to the experiment on adversarial robustness analysed in \cref{se:lipschitz}. We have studied the effect of using different margin values in the loss function adopted for the training of the five different neural networks. In \cref{fig:marginTestRob} we report the results obtained for two additional values of the margin parameter. The observations presented for the case $\mathrm{margin}=0.07$ extend also to these ones, where the dynamically constrained networks still perform better. Moreover, the networks that allow for expansive layers still outperform those with all non-expansive layers.

\red{We additionally provide the details on how the stepsizes for the flexibly constrained neural networks are selected. We recall that the flexible alternation strategy is defined by maps of the form
\begin{align*}
\tilde{\Psi}^{h_1}(x) &= x - h_1 A_c^T\Sigma(A_cx+b_c) =: x-h_1 X(A_c,b_c,x),\,\,A_c^TA_c=I \nonumber \\
\Psi^{h_2}(x) &= x + h_2A^T\mathrm{ReLU}(Ax+b) =: x+ h_2 X(A,b,x),\,\,A^TA=I \\
x & \mapsto \tilde{\Psi}^{h_1/2}\circ \tilde{\Psi}^{h_1/2}\circ \Psi^{h_2}(x)=:\Psi^h(x).
\end{align*}
In order for the map $\Psi^h$ to be non-expansive, we either need to have both $\tilde{\Psi}^{h_1/2}$ and $\Psi^{h_2}$ to be $1-$Lipschitz, or them to be $1-$Lipschitz when composed together. For the former case, this is imposed in our implementation by the constraints $0.11<h_1<1.9$ and $-1.9<h_2\leq 0$, since $\|A_c\|=\|A\|=1$ and the relevant condition is \eqref{eq:contractiveDyn}. In the case $h_2>0$, we need to impose the relation
\[
(h_1,h_2)\in \mathcal{R} = \{(h_1,h_2)\in \mathbb{R}^2:\;(1+h_2)(1-h_1a+h_1^2/4)\leq 1\}.
\]
Because of experimental reasons, we choose to impose it by clamping the $h_1$ timestep, with the PyTorch function $\texttt{clamp}$, so that $0.11<h_1<1.9$. Then, we clamp $h_2$ so that the pair $(h_1,h_2)\in\mathcal{R}$, i.e. in the green area represented in \cref{fig:rectangle}.
\begin{figure}[ht!]
    \centering
    \includegraphics[width=\textwidth]{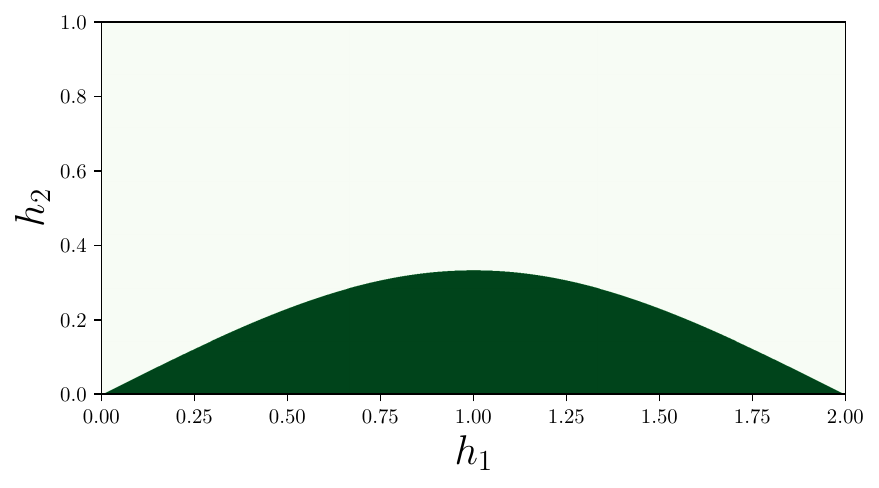}
    \caption{\red{In green, we represent the contractivity area $\mathcal{R}$.}}
    \label{fig:rectangle}
\end{figure}
All the constraints are imposed after each SGD step. In each training iteration, the pair $(h_1,h_2)$ is projected onto the green area.
\begin{figure}
\centering
\begin{subfigure}{.49\textwidth}
\centering
\includegraphics[width=\textwidth]{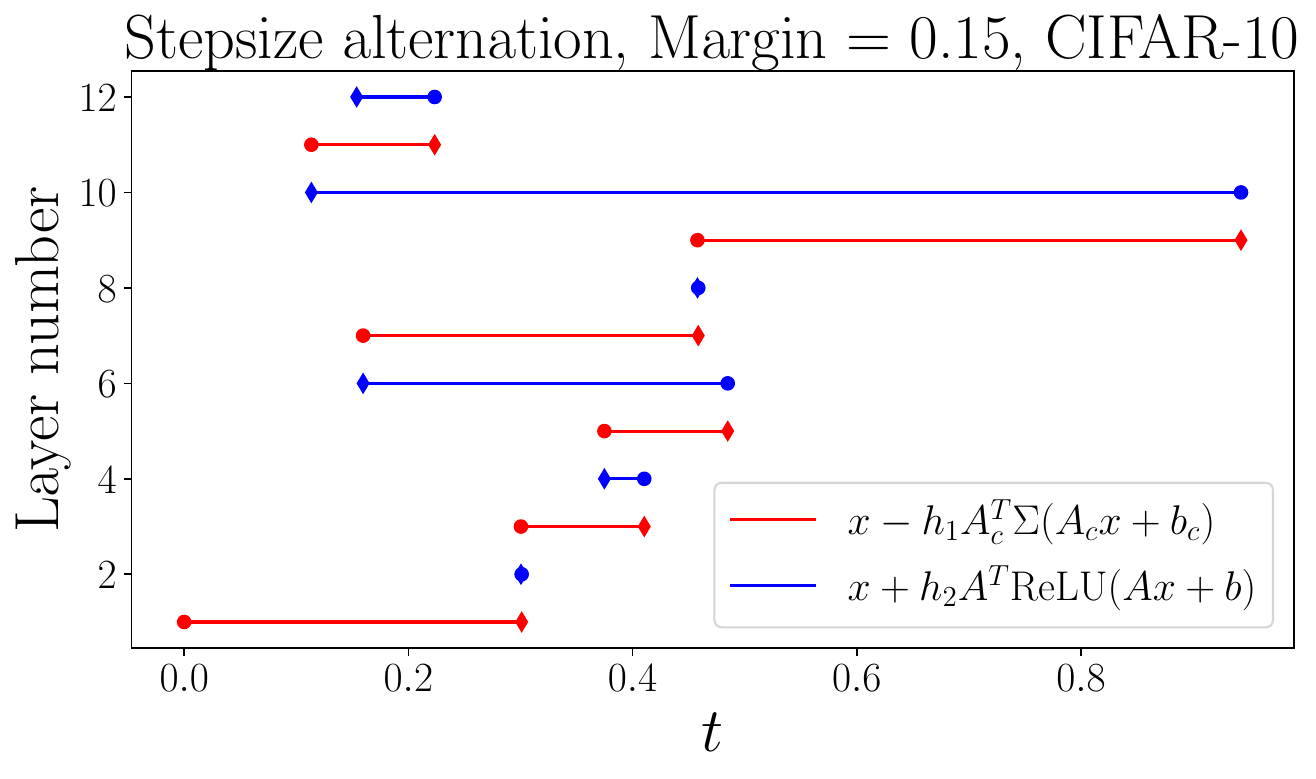}
\caption{Learned step alternation strategy for CIFAR-10 dataset.}
\end{subfigure}
\begin{subfigure}{.49\textwidth}
\centering
\includegraphics[width=\textwidth]{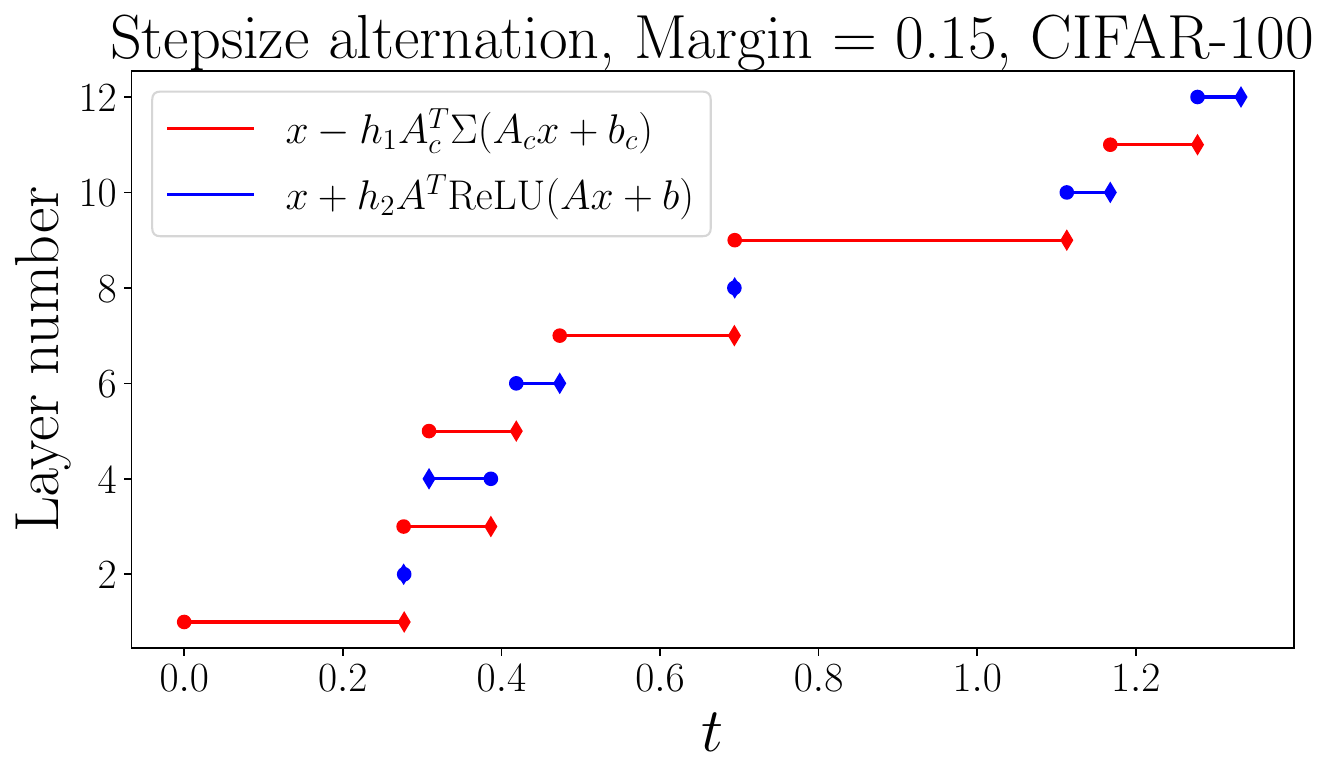}
\caption{Learned step alternation strategy for CIFAR-100 dataset.}
\end{subfigure}
\centering
\begin{subfigure}{.49\textwidth}
\centering
\includegraphics[width=\textwidth]{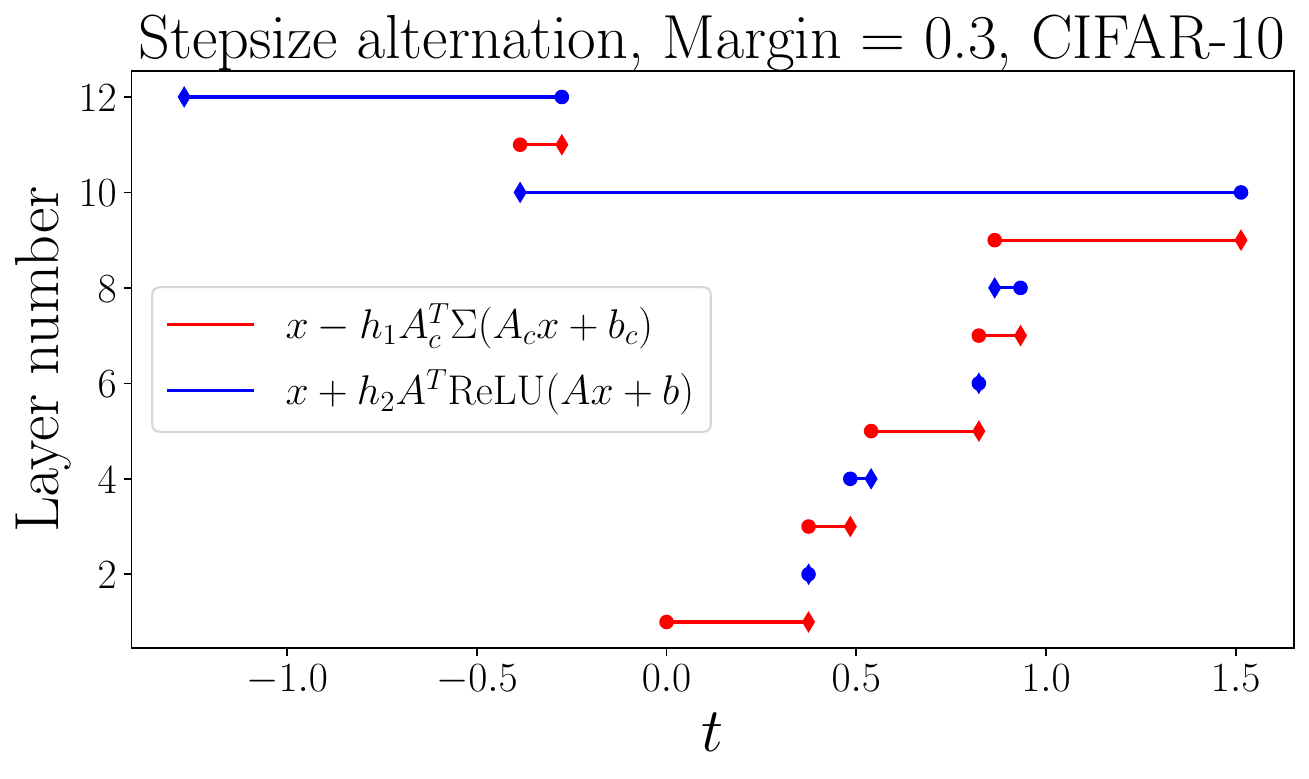}
\caption{Learned step alternation strategy for CIFAR-10 dataset.}
\end{subfigure}
\begin{subfigure}{.49\textwidth}
\centering
\includegraphics[width=\textwidth]{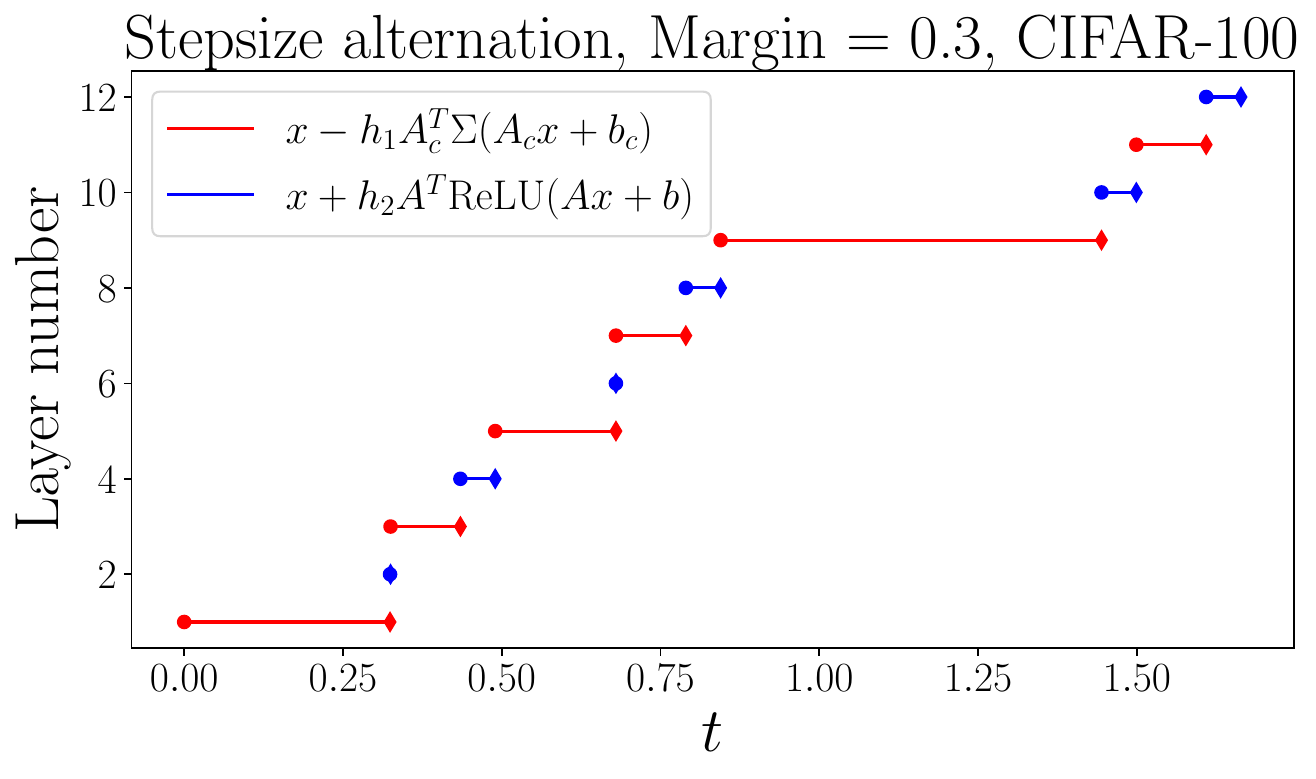}
\caption{Learned step alternation strategy for CIFAR-100 dataset.}
\end{subfigure}
\caption{\red{These figures report the learned step alternation strategies for the flexible regime, with margin values of $0.15$ and $0.3$}}
\label{fig:otherAlternations}
\end{figure}
In \cref{fig:otherAlternations}, we show the learned timesteps for the flexible training strategy that are not included in the manuscript's main text. We remark that \cref{fig:otherAlternations} reports many negative steps for the CIFAR-10 dataset, especially when the margin is $0.15$, where almost all of them are. This pattern does not show up for CIFAR-100, which suggests that the increased complexity of this classification task leads to the need for more freedom in the network layers, given by expansive layers.
}
\section{Proof convergence of splitting strategy for Lipschitz fields}\label{se:convSplitting}
In this section, we prove the convergence of a Lie-Trotter splitting method applied to Lipschitz vector fields, as applied in the proofs of \cref{se:approx}. This reasoning extends similarly to other splitting strategies, like Strang splitting.

\begin{proposition}
Let $X\in\mathfrak{X}(\mathbb{R}^n)$ be a vector field that can be decomposed on the compact set $\Omega\subset\mathbb{R}^n$ as $X=f+g$ for two Lipschitz continuous vector fields $f,g\in\mathfrak{X}(\Omega)$. More precisely, let $L_f,L_g>0$ be such that 
\[
\|f(x)-f(y)\|\leq L_f\|x-y\|,\quad \|g(x)-g(y)\|\leq L_g\|x-y\|\qquad \forall x,y\in\Omega.
\]
The Lie-Trotter splitting method $\varphi^h := \Phi^h_g\circ\Phi^h_f$ is a first-order accurate approximation of the exact flow $\Phi^h_{X}$.
\end{proposition}

\begin{proof}
The proof comes from applying Gronwall's inequality twice. We start considering the function
\[
\gamma(t) := \|\Phi_g^t\circ\Phi_f^h(x_0) - \Phi_{f+g}^t(x_0)\|. 
\]
By the integral definition of the flow map, we have
\[
\Phi_{g}^h\circ \Phi_f^h(x_0) = x_0 + \int_0^hf(\Phi_f^s(x_0))ds + \int_0^h g(\Phi_g^s\circ\Phi_f^h(x_0))ds
\]
and
\[
\Phi_{f+g}^h(x_0) = x_0 + \int_0^hf(\Phi_{f+g}^s(x_0))ds + \int_0^h g(\Phi_{f+g}^s(x_0))ds.
\]
This means that
\begin{align*}
\gamma(h)&\leq L_f\int_0^h \|\Phi_{f+g}^s(x_0)-\Phi_f^s(x_0)\|ds + L_g\int_0^h\|\Phi_g^s\circ \Phi_f^h(x_0)-\Phi_{f+g}^s(x_0)\|ds\\
&= \alpha(h) + \int_0^h \beta(s)\gamma(s)ds
\end{align*}
where $\alpha(h) = L_f\int_0^h \|\Phi_{f+g}^s(x_0)-\Phi_f^s(x_0)\|ds$ and $\beta(s)\equiv L_g$.

Since $\alpha$ is a non-decreasing function, we can apply Gronwall's integral inequality to get
\[
\gamma(h) \leq \alpha(h) \mathrm{exp}\left(\int_0^h\beta(s)ds\right) = \alpha(h)\mathrm{exp}\left(L_gh\right).
\]
We need to bound the function $\alpha(h)$. We study the behaviour of
\[
\lambda(s) := \|\Phi_{f+g}^s(x_0)-\Phi_f^s(x_0)\|
\]
similarly to what was done above. Indeed we have
\[
\Phi_{f+g}^s(x_0) - \Phi_f^s(x_0) = \int_0^sf(\Phi_{f+g}^{s'}(x_0))ds' + \int_0^{s} g(\Phi_{f+g}^{s'}(x_0))ds'-\int_0^sf(\Phi_f^{s'}(x_0))ds'
\]
and hence
\[
\lambda(s) \leq L_f\int_0^s \lambda(s')ds' + \int_0^{s}\|g(\Phi_{f+g}^{s'}(x_0))\|ds'.
\]
Again by Gronwall's inequality, we can conclude
\[
\lambda(s) \leq s\,\max_{x\in\Omega}\|g(x)\|\mathrm{exp}(L_fs).
\]
This inequality allows finishing the proof since
\begin{align*}
\gamma(h) &:= \|\Phi_g^h\circ\Phi_f^h(x_0) - \Phi_{f+g}^h(x_0)\|\\
&\leq \max_{x\in\Omega}\|g(x)\|\mathrm{exp}(L_gh)\int_0^h\mathrm{\exp}(L_fh)sds\\
&=\frac{h^2}{2} \mathrm{\exp}((L_f+L_g)h)\max_{x\in\Omega}\|g(x)\|\\
&\leq \frac{h^2}{2} \mathrm{\exp}(\mathrm{Lip}(X)h)\max_{x\in\Omega}\|g(x)\|.
\end{align*}
Thus Lie-Trotter splitting is a first-order method for the vector field $X$.
\end{proof}

\end{document}